\documentclass[final]{opt2026}

\usepackage{microtype}

\makeatletter
\renewcommand*{\@titlefoot}{}
\makeatother

\allowdisplaybreaks[2]
\newcommand{\R}{\mathbb{R}}
\newcommand{\E}{\mathbb{E}}
\newcommand{\Pp}{\mathbb{P}}
\newcommand{\Law}{\mathcal{L}}
\newcommand{\cN}{\mathcal{N}}
\newcommand{\cF}{\mathcal{F}}
\newcommand{\cC}{\mathcal{C}}
\newcommand{\Id}{I}
\newcommand{\op}{\mathrm{op}}
\newcommand{\tr}{\operatorname{tr}}
\newcommand{\Lip}{\operatorname{Lip}}
\newcommand{\norm}[1]{\lVert#1\rVert}
\newcommand{\ip}[2]{\langle #1,#2\rangle}
\newcommand{\frobip}[2]{\langle #1,#2\rangle_{\mathrm F}}
\newcommand{\ind}{\mathbf{1}}
\newcommand{\dd}{\,\mathrm{d}}
\newcommand{\Wone}{W_1}
\newcommand{\Wjoint}{W_{d_{\alpha,\kappa}}}
\newcommand{\Palpha}{P_\alpha}
\newcommand{\Zspace}{\mathsf{Z}}

\title[Sharp Stationary Gaussian Approximation for Constant-Stepsize SGD]{Sharp Stationary Gaussian Approximation for Constant-Stepsize SGD}

\optauthor{%
\Name{Junghoon Seo}\\
\addr PIT IN Corp., South Korea}

\begin{document}
\maketitle

\begin{abstract}
We prove a sharp Gaussian approximation for the invariant law of constant-stepsize SGD with bounded additive noise generated by an exogenous uniformly ergodic Markov chain. For a smooth, strongly convex objective with a Lipschitz Hessian and nondegenerate long-run noise covariance, the centered iterate normalized by the square root of the stepsize is $O(\sqrt{\alpha})$-close in 1-Wasserstein distance to its limiting Gaussian. The proof combines blockwise Gaussian comparison with long-run contraction. A four-state example gives a matching lower bound although the one-time noise marginal is symmetric and every nonzero-lag autocovariance vanishes. In this example, an adjacent third-order mixed moment produces the leading correction.
\end{abstract}

\section{Introduction}
Constant-stepsize stochastic gradient descent (SGD) is naturally studied through its invariant distribution when the deterministic drift is stable and the gradient noise persists \citep{Pflug1986,DieuleveutDurmusBach2020,ChenMouMaguluri2022}. We consider
\begin{equation}
X_{k+1}=X_k-\alpha\nabla f(X_k)+\alpha\xi(Z_k),
\qquad k\ge0,
\label{eq:sgd}
\end{equation}
where $f:\R^d\to\R$ is smooth and strongly convex, $(Z_k)$ is an exogenous Markov chain with invariant distribution $\pi$, and $\xi:\Zspace\to\R^d$ is centered under $\pi$. Exogeneity means that the transition kernel of $(Z_k)$ is independent of the iterate. At stationarity, $X_k$ and $Z_k$ are generally dependent because the iterate contains the accumulated effect of earlier Markov states.

Let $x_\star$ be the unique minimizer of $f$ and set $H=\nabla^2f(x_\star)$. The scale $\sqrt\alpha$ is suggested by a simple variance calculation. Over $O(\!\alpha^{-1})$ effective steps, the recursion accumulates noise increments of size $\alpha$, so their total standard deviation is of order $\sqrt{\alpha^{-1}}\,\alpha=\sqrt\alpha$. For a stationary chain, temporal dependence changes the covariance through the long-run covariance
\begin{equation}
\Sigma_M
=\E[\xi(Z_0)\xi(Z_0)^\top]
+\sum_{\ell\ge1}\E\!\left[
\xi(Z_0)\xi(Z_\ell)^\top+\xi(Z_\ell)\xi(Z_0)^\top
\right].
\label{eq:long-run-covariance}
\end{equation}
Under the assumptions below, this series converges absolutely. The Gaussian benchmark is $\gamma=\cN(0,\Sigma)$, where
\begin{equation}
H\Sigma+\Sigma H=\Sigma_M.
\label{eq:lyapunov}
\end{equation}
This is the covariance equation for the Ornstein--Uhlenbeck approximation obtained by linearizing the drift at $x_\star$ \citep{Pflug1986,ChenMouMaguluri2022}.

\paragraph{Main contribution.}
\citet{WangEtAl2026} give an $O(\!\sqrt\alpha\log(1/\alpha))$ upper bound for the $\Wone$ distance between the $\alpha^{-1/2}$-scaled centered stationary law and its Gaussian limit. We improve this to $O(\!\sqrt\alpha)$ (Theorem~\ref{thm:main}) and give a four-state example attaining the matching lower bound (Theorem~\ref{thm:lower}).

\paragraph{Proof architecture.}
The proof separates a finite-block Gaussian approximation from long-run contractivity. One block contains $n=\lceil T/\alpha\rceil$ updates, so its deterministic contraction factor is approximately $e^{-mT}$. Within a block, we linearize the drift and use a Poisson equation to decompose the Markovian noise into a martingale difference and a coboundary. The propagated Gaussian initial condition smooths Lipschitz test functions, which permits a second-order Lindeberg expansion. A second Poisson equation telescopes the state-dependent conditional covariance. A coalescent coupling and strong convexity then transfer the block estimate to the invariant law.

\paragraph{Related work.}
Stationary and diffusion descriptions of constant-stepsize stochastic approximation appear in \citet{Pflug1986,MandtHoffmanBlei2017,DieuleveutDurmusBach2020}, and general stationary weak limits are developed by \citet{ChenMouMaguluri2022}. Poisson-equation methods for Markovian stochastic approximation are classical \citep{BenvenisteMetivierPriouret2012,KushnerYin2003,Fort2015}. Quantitative analyses use concentration, Lyapunov methods, coupling, and linear stochastic-approximation techniques \citep{ThoppeBorkar2019,ChenMaguluriShakkottaiShanmugam2023,MouPananjadyWainwrightBartlett2024,DurmusMoulinesNaumovSamsonov2024}. Stationary bias and prelimit coupling under Markovian noise are studied in \citet{HuoChenXie2026,HuoZhangChenXie2024,AllmeierGast2024,MeradGaiffas2025,ZhangHuoChenXie2026}. Dependent-data and finite-horizon Gaussian approximations provide complementary benchmarks \citep{ZhangXie2026,WeiLiLouWu2025,HaqueEtAl2026,KongSrikant2026}. Appendix~\ref{app:related-work} discusses connections with recent Gaussian approximation results~\cite{ZhangXie2026,HaqueEtAl2026,ZhangXieMartingale2026,ZhangXieSteadyState2026}.

\section{Problem setup and main results}
\label{sec:setup}
Fix $d\ge1$. Vector norms and $\ip{\cdot}{\cdot}$ are Euclidean, matrix inequalities use the Loewner order, and $\Lip(h)$ denotes the Lipschitz constant of $h$. For matrices, $\norm{\cdot}_{\op}$, $\norm{\cdot}_F$, and $\norm{\cdot}_*$ denote the operator, Frobenius, and nuclear norms. We write $\Id_d$ for the identity matrix and, for $B,C\in\R^{d\times d}$, $\frobip{B}{C}=\tr(B^\top C)$. The notation $\ind_A$ denotes the indicator of $A$. For a bounded function $g$ taking values in a normed space, let $\norm{g}_\infty=\sup_z\norm{g(z)}$, where the pointwise norm is understood from context and is displayed explicitly for matrices when needed. Unquantified constants denoted by $C$ or $C_T$ are independent of $\alpha$ and may change from line to line. Explicitly quantified constants remain fixed. For probability measures on $\R^d$ with finite first moments, we use the Kantorovich--Rubinstein dual representation as the definition
\begin{equation}
\Wone(\mu,\nu)
=\sup_{\Lip(h)\le1}\left|\int h\dd\mu-\int h\dd\nu\right|.
\end{equation}
Every coupling $(U,V)$ of $(\mu,\nu)$ satisfies $\Wone(\mu,\nu)\le\E\norm{U-V}$ by weak duality \citep{Villani2009}.

\paragraph{Objective assumptions.}
Assume that $f$ is twice continuously differentiable and that there are constants $0<m\le L$ and $0\le M<\infty$ such that
\begin{equation}
m\Id_d\preceq\nabla^2f(x)\preceq L\Id_d,
\qquad
\norm{\nabla^2f(x)-\nabla^2f(y)}_{\op}\le M\norm{x-y}
\label{eq:f-ass}
\end{equation}
for all $x,y\in\R^d$. Thus $f$ is $m$-strongly convex and $L$-smooth, its Hessian is $M$-Lipschitz, and $H=\nabla^2f(x_\star)$ satisfies $m\Id_d\preceq H\preceq L\Id_d$.

\paragraph{Markovian noise assumptions.}
Let $P$ be a Markov kernel on a Polish space $\Zspace$ with invariant probability measure $\pi$. We also use $P$ for the associated Markov operator, $Pg(z)=\int g(u)P(z,\dd u)$. Assume that, for constants $C_Z<\infty$ and $\rho\in(0,1)$,
\begin{equation}
\sup_{z\in\Zspace}\norm{P^k(z,\cdot)-\pi}_{\mathrm{TV}}
\le C_Z\rho^k,
\qquad k\ge0.
\label{eq:uniform-ergodicity}
\end{equation}
Here $\norm{\mu-\nu}_{\mathrm{TV}}=\sup_B|\mu(B)-\nu(B)|$, with the supremum over Borel sets. Assume also that the measurable observable $\xi:\Zspace\to\R^d$ satisfies
\begin{equation}
\int\xi(z)\pi(\dd z)=0,
\qquad
\sup_z\norm{\xi(z)}\le b<\infty.
\label{eq:xi-ass}
\end{equation}
We write $\E_\pi$ for expectation under a stationary $P$-chain. An unsubscripted $\E$ refers to the probability space currently under consideration.
For a stationary $P$-chain, define
\begin{equation}
\Gamma_\ell=\E_\pi[\xi(Z_0)\xi(Z_\ell)^\top],
\quad \ell\ge0,
\qquad
\Gamma_{-\ell}=\Gamma_\ell^\top,
\quad \ell\ge1.
\label{eq:Gamma-def}
\end{equation}
Appendix~\ref{app:poisson} proves
\begin{equation}
C_\Gamma:=\sum_{\ell\in\mathbb Z}\norm{\Gamma_\ell}_*<\infty.
\label{eq:Cgamma}
\end{equation}
Hence $\Sigma_M=\sum_{\ell\in\mathbb Z}\Gamma_\ell$ agrees with \eqref{eq:long-run-covariance}. We assume $\Sigma_M\succ0$. The Lyapunov equation \eqref{eq:lyapunov} then has the unique positive-definite solution
\begin{equation}
\Sigma=\int_0^\infty e^{-tH}\Sigma_Me^{-tH}\dd t,
\qquad \gamma=\cN(0,\Sigma).
\end{equation}
Moreover,
$\lambda_{\min}(\Sigma)\ge\lambda_{\min}(\Sigma_M)/(2L)$ and
$\norm{\Sigma}_{\op}\le\norm{\Sigma_M}_{\op}/(2m)$. Thus the constants below can be chosen with the dependence stated in Theorem~\ref{thm:main}.

\begin{theorem}[Gaussian approximation for the invariant law]
\label{thm:main}
Under \eqref{eq:f-ass}--\eqref{eq:xi-ass} and $\Sigma_M\succ0$, there are constants $\alpha_\star>0$ and $C<\infty$, depending only on $d,m,L,M,b,C_Z,\rho$, and $\Sigma_M$, such that the joint recursion
\[
Z_{k+1}\mid (X_j,Z_j)_{j\le k}\sim P(Z_k,\cdot),
\qquad
X_{k+1}=X_k-\alpha\nabla f(X_k)+\alpha\xi(Z_k)
\]
has a unique invariant probability measure $\Pi_\alpha$ for every $0<\alpha\le\alpha_\star$. Its $Z$-marginal is $\pi$. If $(X_\alpha,Z_\alpha)\sim\Pi_\alpha$, then
\begin{equation}
\Wone\!\left(
\Law\!\left(\frac{X_\alpha-x_\star}{\sqrt\alpha}\right),
\cN(0,\Sigma)
\right)
\le C\sqrt\alpha.
\label{eq:main-bound}
\end{equation}
Consequently,
\begin{equation}
\Wone\bigl(\Law(X_\alpha),\cN(x_\star,\alpha\Sigma)\bigr)
\le C\alpha.
\label{eq:main-unscaled}
\end{equation}
\end{theorem}

\begin{theorem}[Matching lower bound from temporal dependence]
\label{thm:lower}
There are constants $c_0,\alpha_0>0$ and a scalar instance satisfying \eqref{eq:f-ass}--\eqref{eq:xi-ass} such that, for the stationary scaled iterate $Y_\alpha=(X_\alpha-x_\star)/\sqrt\alpha$,
\begin{equation}
\Wone\bigl(\Law(Y_\alpha),\cN(0,5/4)\bigr)
\ge c_0\sqrt\alpha,
\qquad 0<\alpha\le\alpha_0.
\label{eq:lower-main}
\end{equation}
The stationary noise marginal is symmetric, $\E[\xi(Z_0)^3]=0$, and $\E[\xi(Z_0)\xi(Z_\ell)]=0$ for every $\ell\in\mathbb Z\setminus\{0\}$. The leading correction is generated by the mixed moment $\E[\xi(Z_0)\xi(Z_1)^2]=-9/4$. Independent sampling from the same marginal gives zero expectation for the bounded $1$-Lipschitz test function $h(x)=\sin x$.
\end{theorem}

The construction takes $Z_k=(\eta_{k-1},\eta_k)$ for independent Rademacher variables and $\xi(Z_k)=\eta_k(3-\eta_{k-1})/2$. Adjacent noise variables share $\eta_k$ and are dependent. Appendix~\ref{app:lower} derives the stationary weighted series, computes its exact second and third moments, and obtains
$\E[\sin(Y_\alpha)]=\tfrac38e^{-5/8}\sqrt\alpha+O(\alpha)$ through an elementary two-state characteristic-function recursion. Since $\sin$ is $1$-Lipschitz and every centered Gaussian law is symmetric, this expansion proves \eqref{eq:lower-main}.

\section{Block approximation and stationary transfer}
Define the scaled iterate and its deterministic update map by
\begin{align}
Y_k&=\frac{X_k-x_\star}{\sqrt\alpha},
&\mathcal T_\alpha(y)&=y-\sqrt\alpha\nabla f(x_\star+\sqrt\alpha y),
\\
Y_{k+1}&=\mathcal T_\alpha(Y_k)+\sqrt\alpha\xi(Z_k).
\label{eq:scaled-rec}
\end{align}
For $\alpha\le L^{-1}$, the mean-value theorem and \eqref{eq:f-ass} give
\begin{equation}
\norm{\mathcal T_\alpha(y)-\mathcal T_\alpha(y')}
\le(1-m\alpha)\norm{y-y'}.
\label{eq:T-contract}
\end{equation}
Thus one step has a contraction gap of order $\alpha$. Taking $n=\lceil T/\alpha\rceil$ steps turns this into a fixed contraction factor of order $e^{-mT}$.

Let $\Palpha$ be the transition kernel of $(Y_k,Z_k)$, also use $\Palpha$ for its Markov operator on functions, write $\lambda\Palpha^n$ for the law after $n$ steps from $\lambda$, and let $\nu=\gamma\otimes\pi$. The next propositions isolate the block approximation and joint stability. Appendix~\ref{app:reader-guide} gives a proof guide, followed by the formal arguments in Appendices~\ref{app:block-proof} and~\ref{app:joint-proof}.

\begin{proposition}[One-block Gaussian approximation]
\label{prop:block-defect}
For every fixed $T>0$, there are $C_T<\infty$ and $\alpha_T>0$ such that, if $(Y_0,Z_0)\sim\gamma\otimes\pi$ and $n=\lceil T/\alpha\rceil$, then
\begin{equation}
\Wone(\Law(Y_n),\gamma)\le C_T\sqrt\alpha,
\qquad 0<\alpha\le\alpha_T.
\end{equation}
\end{proposition}

For $\kappa>0$, equip $\R^d\times\Zspace$ with the cost
\begin{equation}
d_{\alpha,\kappa}((y,z),(y',z'))
=\norm{y-y'}+\kappa\sqrt\alpha\ind_{\{z\ne z'\}}.
\end{equation}
For product-space laws $\lambda$ and $\lambda'$ with finite first $Y$-moments, $\Wjoint(\lambda,\lambda')$ is the supremum of $|\int h\dd\lambda-\int h\dd\lambda'|$ over measurable functions $h:\R^d\times\Zspace\to\R$ satisfying
$|h(y,z)-h(y',z')|\le d_{\alpha,\kappa}((y,z),(y',z'))$.
Every coupling upper-bounds this distance by its expected cost. The factor $\sqrt\alpha$ matches the size of one noise increment in \eqref{eq:scaled-rec}.

\begin{proposition}[One-block contraction of the joint process]
\label{prop:block-contract}
There are $T>0$, $0<\kappa<\infty$, $q\in(0,1)$, and $\alpha_c>0$ such that, with $n=\lceil T/\alpha\rceil$,
\begin{equation}
\Wjoint(\lambda\Palpha^n,\lambda'\Palpha^n)
\le q\Wjoint(\lambda,\lambda')
\end{equation}
for all probability measures with finite first $Y$-moments and every $0<\alpha\le\alpha_c$.
\end{proposition}

The one-block approximation uses two Poisson equations. We state them here because they identify the two places where serial dependence enters the proof.

\begin{lemma}[Noise and covariance correctors]
\label{lem:poisson}
Let $\cF_k=\sigma(Z_0,\ldots,Z_k)$ and define
\begin{align}
V&=\sum_{j\ge0}P^j\xi,
\\
D_{k+1}&=V(Z_{k+1})-PV(Z_k),
\label{eq:D-def}\\
\cC(z)&=P(VV^\top)(z)-(PV)(z)(PV)(z)^\top,
\\
W&=\sum_{j\ge0}P^j(\cC-\Sigma_M).
\end{align}
Both series converge uniformly, with nuclear-norm convergence for $W$. Moreover, $V-PV=\xi$, $(D_{k+1})$ is a bounded martingale difference sequence, and
\begin{equation}
\E[D_{k+1}D_{k+1}^\top\mid\cF_k]=\cC(Z_k),
\qquad
\int\cC\dd\pi=\Sigma_M.
\end{equation}
The exact identities
\begin{equation}
\xi(Z_k)=D_{k+1}+V(Z_k)-V(Z_{k+1}),
\qquad
W-PW=\cC-\Sigma_M
\label{eq:mart-cob}
\end{equation}
will be used for summation by parts and covariance telescoping.
\end{lemma}

The first corrector turns the Markov sum into a martingale sum plus a bounded coboundary remainder. The second turns predictable covariance fluctuations into an exact telescoping sum. These are the two points at which temporal dependence changes the Gaussian replacement. Appendix~\ref{app:joint-proof} constructs the invariant law and gives the complete stationary-transfer argument.

\section{Proof of Theorem~\ref{thm:main}}
Fix the block duration from Proposition~\ref{prop:block-contract}, let $n=\lceil T/\alpha\rceil$, and start from $(Y_0,Z_0)\sim\gamma\otimes\pi$. We follow the comparison path and record the order of each error.

\paragraph{Step 1: linearize the drift.}
Set $A=\Id_d-\alpha H$. This matrix is symmetric. Let
\[
\bar Y_{k+1}=A\bar Y_k+\sqrt\alpha\xi(Z_k),
\qquad \bar Y_0=Y_0.
\]
The Lipschitz Hessian condition and Taylor's formula give
\begin{equation}
\norm{\mathcal T_\alpha(y)-Ay}
\le\frac{M}{2}\alpha^{3/2}\norm{y}^2.
\label{eq:main-linearization}
\end{equation}
Stationarity of the Markov chain and absolute summability of $(\Gamma_\ell)$ imply
\[
\sup_{0\le k\le n}\E\norm{\bar Y_k}^2
\le\tr(\Sigma)+(T+1)C_\Gamma.
\]
Coupling $Y_k$ and $\bar Y_k$ with the same Markov path, then iterating \eqref{eq:T-contract} and \eqref{eq:main-linearization}, yields
\begin{equation}
\E\norm{Y_n-\bar Y_n}\le C_T\sqrt\alpha.
\label{eq:main-nonlinear}
\end{equation}
\paragraph{Step 2: reduce the Markov sum to a martingale sum.}
For $r=1,\ldots,n$, write $B_r=\sqrt\alpha A^{n-r}$. Substituting the first identity in \eqref{eq:mart-cob} into the explicit linear recursion and applying weighted summation by parts gives
\begin{equation}
\bar Y_n=A^nY_0+\sum_{r=1}^nB_rD_r+R_{V,n},
\end{equation}
where
\[
R_{V,n}=\sqrt\alpha\left[
A^{n-1}V(Z_0)-V(Z_n)
+\alpha\sum_{j=1}^{n-1}A^{n-1-j}HV(Z_j)
\right].
\]
The first two terms are endpoints, and the sum records the variation of the matrix weights. Since $V$ is bounded, $\norm{A^j}_{\op}\le1$, and $\alpha n\le T+1$,
\begin{equation}
\norm{R_{V,n}}\le C_T\sqrt\alpha.
\label{eq:main-coboundary}
\end{equation}

\paragraph{Step 3: replace the martingale increments by Gaussian increments.}
Let $N_r\stackrel{\mathrm{iid}}{\sim}\cN(0,\Sigma_M)$ be independent of $(Y_0,(Z_k))$. After restricting $\alpha\le(2L)^{-1}$, $\operatorname{Cov}(A^nY_0)\succeq c_T\Id_d$ for some $c_T>0$ independent of $\alpha$. Hence the function $\psi(u)=\E[h(A^nY_0+u)]$ has a bounded and Lipschitz Hessian for every $1$-Lipschitz $h$.

For the $r$th replacement, let $U_r=\sum_{1\le i<r}B_iD_i+\sum_{r<i\le n}B_iN_i$. Conditional centering cancels the linear terms, while the Taylor remainders total $O(n\alpha^{3/2})=O(\sqrt\alpha)$. With $K_r=B_r^\top\nabla^2\psi(U_r)B_r$, the second Poisson equation gives
\begin{align*}
\sum_{r=1}^n\E\frobip{K_r}{\cC(Z_{r-1})-\Sigma_M}
&=\E\frobip{K_1}{W(Z_0)}-\E\frobip{K_n}{W(Z_n)}\\
&\quad+\sum_{r=1}^{n-1}\E\frobip{K_{r+1}-K_r}{W(Z_r)}.
\end{align*}
The boundary terms are $O(\alpha)$, while $\E\norm{K_{r+1}-K_r}_{\op}=O(\alpha^{3/2}+\alpha^2)$. Thus the variation sum is $O(\sqrt\alpha)$, and
\begin{equation}
\Wone\!\left(
\Law\!\left(A^nY_0+\sum_{r=1}^nB_rD_r\right),
\Law\!\left(A^nY_0+\sum_{r=1}^nB_rN_r\right)
\right)
\le C_T\sqrt\alpha.
\label{eq:main-lindeberg}
\end{equation}

\paragraph{Step 4: identify the target covariance.}
The Gaussian vector in \eqref{eq:main-lindeberg} has covariance $\Sigma_{\alpha,n}=A^n\Sigma A^n+\alpha\sum_{j=0}^{n-1}A^j\Sigma_MA^j$. The identities $\Sigma_M=H\Sigma+\Sigma H$ and $A=\Id_d-\alpha H$ give
\begin{equation}
\Sigma_{\alpha,n}-\Sigma
=\alpha^2\sum_{j=0}^{n-1}A^jH\Sigma H A^j,
\qquad
\norm{\Sigma_{\alpha,n}-\Sigma}_F\le C_T\alpha.
\end{equation}
Appendix~\ref{app:smoothing} converts this covariance error into an $O(\alpha)$ Wasserstein error. Equations~\eqref{eq:main-nonlinear}, \eqref{eq:main-coboundary}, and \eqref{eq:main-lindeberg} then prove Proposition~\ref{prop:block-defect}.

\paragraph{Step 5: contract the joint process and transfer to stationarity.}
Uniform ergodicity gives a coupling with coalescence time $\tau$ satisfying $\Pp(\tau>k)\le C_{\mathrm{cpl}}\rho_{\mathrm{cpl}}^k$. Under this coupling, strong convexity contracts the iterate distance by $1-m\alpha$ after the Markov paths meet, while bounded noise adds at most $2b\sqrt\alpha$ before meeting. Choosing $T$ and $\kappa$ therefore makes the $n$-step joint kernel contract by a fixed factor $q<1$. Appendix~\ref{app:joint-proof} constructs its unique invariant law $\varpi_\alpha$ and proves $(1-q)\Wjoint(\varpi_\alpha,\nu)\le(C_T+2\kappa)\sqrt\alpha$. Projection gives \eqref{eq:main-bound}, and undoing the scaling gives \eqref{eq:main-unscaled}.

\bibliography{references}

\clearpage
\appendix

\section{Connections with recent Gaussian approximation results}
\label{app:related-work}

The CLT for time-varying observables in \citet[Theorem~3]{ZhangXie2026} offers an alternative to our direct Gaussian comparison for the linearized block. Fix $T>0$, let $n=\lceil T/\alpha\rceil$, and set $A=\Id_d-\alpha H$. The noise contribution to this block is
\begin{align*}
S_{\alpha,n}
&=\sqrt\alpha\sum_{k=0}^{n-1}A^{n-1-k}\xi(Z_k)
=\frac1{\sqrt n}\sum_{k=0}^{n-1}h_{\alpha,k}(Z_k),\\
h_{\alpha,k}(z)&=\sqrt{n\alpha}\,A^{n-1-k}\xi(z).
\end{align*}
For $0<\alpha\le\min\{1,L^{-1}\}$, these centered observables satisfy
$\sup_{0\le k<n}\norm{h_{\alpha,k}}_\infty\le\sqrt{T+1}\,b$.
Under stationarity, geometric covariance decay gives
\[
\operatorname{Cov}(S_{\alpha,n})
=\alpha\sum_{j=0}^{n-1}A^j\Sigma_MA^j+O(\alpha)
\longrightarrow
\int_0^T e^{-tH}\Sigma_Me^{-tH}\dd t\succ0
\quad\text{as }\alpha\downarrow0,
\]
where the error is measured in operator norm. For the fixed chain and horizon, these bounds make the CLT constants uniform for sufficiently small $\alpha$, yielding an $O(n^{-1/2})=O(\sqrt\alpha)$ Wasserstein-1 comparison with $\cN(0,\operatorname{Cov}(S_{\alpha,n}))$. The Gaussian comparator for $A^nY_0+S_{\alpha,n}$, with $Y_0\sim\gamma$ independent of the chain, has covariance $\Sigma+O(\alpha)$ and is $O(\alpha)$-close to $\gamma$ in Wasserstein-1 distance. The nonlinear comparison and joint contraction then yield our stationary bound. Our proof establishes the block comparison directly by Poisson--Lindeberg interpolation.

\citet[Proposition~4.5]{HaqueEtAl2026} obtain a finite-time lower bound for linear stochastic approximation from a nonzero third moment of an i.i.d. noise projection. Our matching stationary lower bound instead isolates temporal dependence: the noise marginal is symmetric and all nonzero-lag autocovariances vanish, yet $\E[\xi(Z_0)\xi(Z_1)^2]=-9/4$ produces the leading order-$\sqrt\alpha$ correction.

Concurrent work by \citet{ZhangXieMartingale2026,ZhangXieSteadyState2026} develops sharp Gaussian bounds for martingale sums from uniformly ergodic Markov chains and for contractive stochastic approximation with multiplicative Markov noise. Under local quadratic linearization, their stationary result gives the optimal $O(\sqrt\alpha)$ rate in Wasserstein-2 distance for this broader class of recursions. By H\"older's inequality, $\Wone(\mu,\nu)\le W_2(\mu,\nu)$ for the same pair of laws, so their stationary result also implies the $O(\sqrt\alpha)$ rate in Wasserstein-1 distance. For SGD with bounded additive Markov noise, our direct Wasserstein-1 proof uses a Gaussian initial condition for smoothing and Poisson correctors for the noise and covariance fluctuations. The four-state example above gives a matching $\Omega(\sqrt\alpha)$ lower bound driven by higher-order temporal dependence, even with zero stationary mean bias and vanishing nonzero-lag noise autocovariances.

\section{Reader's guide to the upper-bound proof}
\label{app:reader-guide}

This section gives a calculation-level roadmap for the upper bound. It explains the block length, the two Poisson equations, the Lindeberg interpolation, and the transfer from one block to stationarity. The formal estimates appear in Appendices~\ref{app:poisson}--\ref{app:smoothing}.

\subsection{Block length and the linear reference recursion}

After centering and scaling, the update is
\[
Y_{k+1}=\mathcal T_\alpha(Y_k)+\sqrt\alpha\,\xi(Z_k).
\]
For two points $y,y'$, the fundamental theorem of calculus gives
\[
\mathcal T_\alpha(y)-\mathcal T_\alpha(y')
=\left[\Id_d-\alpha\int_0^1
\nabla^2 f\bigl(x_\star+\sqrt\alpha(y'+t(y-y'))\bigr)\dd t\right](y-y').
\]
The averaged Hessian in brackets has eigenvalues in $[m,L]$. When $\alpha\le L^{-1}$, the bracketed matrix therefore has operator norm at most $1-m\alpha$. This proves the one-step contraction in \eqref{eq:T-contract}.

Linearizing the drift at $x_\star$ gives
\[
A=\Id_d-\alpha H,
\qquad
\bar Y_{k+1}=A\bar Y_k+\sqrt\alpha\,\xi(Z_k).
\]
The eigenvalues of $A$ lie in $[0,1-m\alpha]$, so
\[
\norm{A^j}_{\op}\le (1-m\alpha)^j\le e^{-m\alpha j}.
\]
One update has only an order-$\alpha$ contraction gap. A block of
$n=\lceil T/\alpha\rceil$ updates satisfies
\[
\norm{A^n}_{\op}\le e^{-m\alpha n}\le e^{-mT},
\]
which is a fixed contraction factor once $T$ is fixed.

Iterating the linear recursion gives
\[
\bar Y_n=A^nY_0+\sqrt\alpha\sum_{k=0}^{n-1}A^{n-1-k}\xi(Z_k).
\]
After restricting $\alpha\le(2L)^{-1}$, the first term remains a nondegenerate Gaussian vector over a fixed block when $Y_0\sim\gamma$. The second term is a matrix-weighted additive functional of the Markov chain.

\subsection{The first Poisson equation removes the conditional mean}

For a centered observable,
\[
P^j\xi(z)=\E[\xi(Z_j)\mid Z_0=z].
\]
Uniform ergodicity makes this conditional mean decay geometrically. Hence
\[
V=\sum_{j\ge0}P^j\xi
\]
converges uniformly and satisfies $V-PV=\xi$. Define
\[
D_{k+1}=V(Z_{k+1})-PV(Z_k).
\]
The Markov property yields $\E[D_{k+1}\mid\cF_k]=0$, and the Poisson equation gives
\[
\xi(Z_k)=D_{k+1}+V(Z_k)-V(Z_{k+1}).
\]
Thus the additive functional is a martingale sum plus a coboundary.

For the weights
\[
B_r=\sqrt\alpha A^{n-r},
\]
write $V_r=V(Z_r)$. Expanding the weighted coboundary sum and collecting the coefficient of each interior $V_r$ gives
\begin{align*}
\sum_{r=1}^{n}B_r(V_{r-1}-V_r)
&=B_1V_0-B_nV_n
+\sum_{r=1}^{n-1}(B_{r+1}-B_r)V_r.
\end{align*}
This is the discrete summation-by-parts identity used in the proof. The endpoint weights have operator norm at most $\sqrt\alpha$, while
\[
B_{r+1}-B_r
=\alpha^{3/2}A^{n-r-1}H.
\]
The variation of one weight is therefore $O(\alpha^{3/2})$. There are $O(\alpha^{-1})$ weights in a block, so boundedness of $V$ makes the entire coboundary correction $O(\sqrt\alpha)$.

\subsection{Gaussian smoothing and Lindeberg interpolation}

After the first Poisson equation, the leading variable is
\[
A^nY_0+\sum_{r=1}^{n}B_rD_r.
\]
Let $N_r$ be independent $\cN(0,\Sigma_M)$ vectors. The Lindeberg interpolation replaces $B_rD_r$ by $B_rN_r$ one index at a time.

The propagated initial condition provides the smoothing needed for a second-order Taylor expansion. Under the restriction $\alpha\le(2L)^{-1}$, its covariance is $A^n\Sigma A^n$, which is bounded below by $c_T\Id_d$ for a constant $c_T>0$ independent of $\alpha$. For a $1$-Lipschitz test function $h$, define
\[
\psi(u)=\E[h(A^nY_0+u)].
\]
Gaussian convolution gives a bounded Hessian and a Lipschitz Hessian for $\psi$.

At position $r$, set
\[
U_r=\sum_{i<r}B_iD_i+\sum_{i>r}B_iN_i,
\qquad
\mathcal G_r=\cF_{r-1}\vee\sigma(N_i:i>r).
\]
Then $U_r$ is $\mathcal G_r$-measurable. The future Gaussian variables are independent of the Markov chain, so
\[
\E[D_r\mid\mathcal G_r]=0,
\qquad
\E[D_rD_r^\top\mid\mathcal G_r]=\cC(Z_{r-1}).
\]
The Gaussian candidate $N_r$ is conditionally centered with covariance $\Sigma_M$. The linear terms in the two Taylor expansions cancel after conditioning. Since $\norm{B_r}_{\op}\le\sqrt\alpha$, each cubic remainder is $O(\alpha^{3/2})$. Summing over $n=O(\alpha^{-1})$ positions gives an $O(\sqrt\alpha)$ contribution.

\subsection{The second Poisson equation removes covariance fluctuations}

The quadratic Taylor term contains
\[
\cC(Z_{r-1})
=\E[D_rD_r^\top\mid\cF_{r-1}].
\]
The first Poisson equation is constructed so that
\[
\int\cC\dd\pi=\Sigma_M.
\]
Hence $\cC-\Sigma_M$ is centered under $\pi$. Uniform ergodicity gives the bounded matrix corrector
\[
W=\sum_{j\ge0}P^j(\cC-\Sigma_M),
\qquad
W-PW=\cC-\Sigma_M.
\]

Let
\[
K_r=B_r^\top\nabla^2\psi(U_r)B_r.
\]
Because $K_r$ is $\mathcal G_r$-measurable, the Markov property gives
\[
\E\frobip{K_r}{PW(Z_{r-1})}
=\E\frobip{K_r}{W(Z_r)}.
\]
Therefore
\begin{align*}
&\sum_{r=1}^{n}\E\frobip{K_r}{\cC(Z_{r-1})-\Sigma_M}\notag\\
&\quad=\sum_{r=1}^{n}\E\frobip{K_r}{W(Z_{r-1})}
-\sum_{r=1}^{n}\E\frobip{K_r}{W(Z_r)}\\
&\quad=\E\frobip{K_1}{W(Z_0)}-\E\frobip{K_n}{W(Z_n)}
+\sum_{r=1}^{n-1}\E\frobip{K_{r+1}-K_r}{W(Z_r)}.
\end{align*}
The second equality follows by shifting the index in the second sum. This is the stochastic summation-by-parts step.

The bounded Hessian gives $\norm{K_r}_{\op}=O(\alpha)$, so the two boundary terms are $O(\alpha)$. To control the variation, write $J_r=\nabla^2\psi(U_r)$ and decompose
\begin{align*}
K_{r+1}-K_r
&=(B_{r+1}-B_r)^\top J_{r+1}B_{r+1}
+B_r^\top(J_{r+1}-J_r)B_{r+1}\\
&\quad+B_r^\top J_r(B_{r+1}-B_r).
\end{align*}
The first and third terms are $O(\alpha^2)$. The middle term is $O(\alpha^{3/2})$ in expectation because the Hessian is Lipschitz and
\[
\E\norm{U_{r+1}-U_r}=O(\sqrt\alpha).
\]
Summing $O(\alpha^{-1})$ variation terms gives $O(\sqrt\alpha)$.

\subsection{Target covariance and transfer to stationarity}

After all replacements, the comparison variable is Gaussian with covariance
\[
\Sigma_{\alpha,n}
=A^n\Sigma A^n+\alpha\sum_{j=0}^{n-1}A^j\Sigma_MA^j.
\]
The continuous Lyapunov equation and $A=\Id_d-\alpha H$ give the one-step identity
\[
A\Sigma A+\alpha\Sigma_M
=\Sigma+\alpha^2H\Sigma H.
\]
Multiplying by $A^j$ on the left and right and summing over $j$ makes the left-hand side telescope. Hence
\[
\Sigma_{\alpha,n}-\Sigma
=\alpha^2\sum_{j=0}^{n-1}A^jH\Sigma H A^j.
\]
Since $n=O(\alpha^{-1})$, the covariance error is $O(\alpha)$.

Let $\varpi_\alpha$ be the invariant law of the joint scaled process and let $\nu=\gamma\otimes\pi$. The block coupling gives a fixed contraction factor $q<1$ in the joint transportation distance. The $Z$-marginal of both $\nu\Palpha^n$ and $\nu$ is $\pi$. For an admissible joint test, replacing the state coordinate by a fixed state changes each expectation by at most $\kappa\sqrt\alpha$. Proposition~\ref{prop:block-defect} therefore implies
\[
\Wjoint(\nu\Palpha^n,\nu)
\le \Wone(\Law(Y_n),\gamma)+2\kappa\sqrt\alpha
=O(\sqrt\alpha).
\]
Invariance and the triangle inequality now give
\[
\Wjoint(\varpi_\alpha,\nu)
\le q\Wjoint(\varpi_\alpha,\nu)
+\Wjoint(\nu\Palpha^n,\nu).
\]
Because $1-q$ is a fixed positive number, rearranging preserves the $O(\sqrt\alpha)$ order.

For reference, the estimates used in one block are summarized below.
\begin{center}
\begin{tabular}{p{0.27\textwidth}p{0.43\textwidth}p{0.13\textwidth}}
\hline
Source & Scale calculation & Block error\\
\hline
Nonlinear drift & $\alpha^{3/2}\sum_{k=0}^{n-1}\E\norm{\bar Y_k}^2$ & $O(\sqrt\alpha)$\\
First Poisson correction & $\sqrt\alpha+ n\alpha^{3/2}$ & $O(\sqrt\alpha)$\\
Lindeberg remainders & $n\alpha^{3/2}$ & $O(\sqrt\alpha)$\\
Conditional covariance & $\alpha+n(\alpha^{3/2}+\alpha^2)$ & $O(\sqrt\alpha)$\\
Discrete covariance & $n\alpha^2$ & $O(\alpha)$\\
\hline
\end{tabular}
\end{center}

\section{Consequences of uniform ergodicity}
\label{app:poisson}

This appendix develops the Markov-chain estimates used throughout the proof. For a centered function $g$, the quantity $P^kg(z)$ is the conditional mean of $g(Z_k)$ when the chain starts from $z$. Uniform ergodicity makes this conditional mean decay geometrically. That estimate yields three consequences: a coalescent coupling with a geometric tail, absolute summability of the stationary covariance sequence, and bounded solutions of the two Poisson equations in Lemma~\ref{lem:poisson}. For a bounded function $g$ taking values in a finite-dimensional normed space, recall the supremum-norm convention from Section~\ref{sec:setup} and write
\[
\pi g=\int g(z)\pi(\dd z).
\]

\begin{lemma}[Uniform decay of centered observables]
\label{lem:mixing-decay}
Let $g:\Zspace\to E$ be bounded and measurable, where $E$ is a finite-dimensional normed space, and suppose $\pi g=0$. Under \eqref{eq:uniform-ergodicity},
\begin{equation}
\sup_{z\in\Zspace}\norm{P^kg(z)}
\le 2C_Z\rho^k\norm{g}_\infty,
\qquad k\ge0.
\end{equation}
\end{lemma}

\begin{proof}
Fix $z\in\Zspace$. The centering condition $\pi g=0$ gives
\[
P^kg(z)=\int g(u)\,[P^k(z,\dd u)-\pi(\dd u)].
\]
Let $E^\ast$ be the dual space of $E$, with dual norm denoted by $\norm{\cdot}_{E^\ast}$. The dual representation of the norm yields
\begin{align*}
\norm{P^kg(z)}
&=\sup_{\substack{\ell\in E^\ast\\ \norm{\ell}_{E^\ast}\le1}}
\left|\int \ell(g(u))\,[P^k(z,\dd u)-\pi(\dd u)]\right|.
\end{align*}
For every functional in this supremum, $|\ell(g(u))|\le\norm{g}_\infty$. With the total-variation convention fixed above, integration against a bounded scalar function satisfies
\[
\left|\int q\,\dd(\mu-\nu)\right|
\le2\norm{q}_\infty\norm{\mu-\nu}_{\mathrm{TV}}.
\]
Consequently,
\[
\norm{P^kg(z)}
\le2\norm{g}_\infty
\norm{P^k(z,\cdot)-\pi}_{\mathrm{TV}}.
\]
Equation~\eqref{eq:uniform-ergodicity} now gives the stated bound uniformly in $z$.
\end{proof}

\begin{lemma}[Geometric coalescent coupling from uniform ergodicity]
\label{lem:uniform-coupling}
Under \eqref{eq:uniform-ergodicity}, there exist constants $C_{\mathrm{cpl}}<\infty$ and $\rho_{\mathrm{cpl}}\in(0,1)$, depending only on $C_Z$ and $\rho$, and a coupling kernel measurable in the initial pair $(z,z')\in\Zspace^2$. Under this kernel, the two $P$-chains have coalescence time
\[
\tau=\inf\{k\ge0:Z_j=Z'_j\text{ for every }j\ge k\}
\]
satisfying
\begin{equation}
\Pp_{z,z'}(\tau>k)
\le C_{\mathrm{cpl}}\rho_{\mathrm{cpl}}^k
\ind_{\{z\ne z'\}},
\qquad k\ge0.
\label{eq:derived-coupling}
\end{equation}
Here $\Pp_{z,z'}$ denotes probability under the coupling kernel started from $(z,z')$. The kernel can be chosen Markovian at the boundaries of fixed-length blocks, which is the structure used in the joint contraction argument.
\end{lemma}

\begin{proof}
The construction uses the $r$-step skeleton of the chain. Choose an integer $r\ge1$ such that
\[
2C_Z\rho^r\le\frac12.
\]
For any two initial states $z,z'\in\Zspace$, the triangle inequality and uniform ergodicity imply
\begin{align*}
\norm{P^r(z,\cdot)-P^r(z',\cdot)}_{\mathrm{TV}}
&\le \norm{P^r(z,\cdot)-\pi}_{\mathrm{TV}}
   +\norm{P^r(z',\cdot)-\pi}_{\mathrm{TV}}\\
&\le2C_Z\rho^r
\le\frac12.
\end{align*}
Thus the two $r$-step transition laws overlap uniformly. At each block boundary, use a maximal coupling for the two endpoints at the next boundary. Their conditional disagreement probability is at most $1/2$. On a Polish space, the maximal coupling can be selected measurably in the initial pair \citep{Griffeath1975,DoucMoulinesPriouretSoulier2018,MeynTweedie2009}.

Apply this endpoint coupling successively at times $0,r,2r,\ldots$. The iterate recursion also uses the intermediate states inside each block. To retain the correct path law, sample those states from the regular conditional law of a Markov path given its initial and terminal states. Such bridge laws exist on a Polish space \citep{Faden1985}. Once the two endpoints agree at a block boundary, use the same transitions in both coordinates thereafter. The two paths then remain equal.

If $z=z'$, this synchronous construction starts at time zero. If $z\ne z'$, every block before coalescence has conditional meeting probability at least $1/2$. Therefore
\[
\Pp_{z,z'}(\tau>jr)\le2^{-j},
\qquad j\ge0.
\]
For an arbitrary integer $k\ge0$, the event $\{\tau>k\}$ is contained in the event that the chains have not coalesced by the preceding block boundary. Hence
\[
\Pp_{z,z'}(\tau>k)
\le2^{-\lfloor k/r\rfloor}
\le2\bigl(2^{-1/r}\bigr)^k.
\]
This proves \eqref{eq:derived-coupling} with
$C_{\mathrm{cpl}}=2$ and $\rho_{\mathrm{cpl}}=2^{-1/r}$.
\end{proof}

\begin{lemma}[Absolute covariance summability]
Let $\Gamma_\ell$ be the stationary covariance matrices defined in \eqref{eq:Gamma-def}. Under \eqref{eq:uniform-ergodicity}--\eqref{eq:xi-ass},
\begin{equation}
\norm{\Gamma_\ell}_*
\le2b^2C_Z\rho^\ell,
\qquad \ell\ge1.
\end{equation}
Consequently, the constant $C_\Gamma$ in \eqref{eq:Cgamma} is finite, and the long-run covariance series in \eqref{eq:long-run-covariance} converges absolutely in nuclear norm.
\end{lemma}

\begin{proof}
By the Markov property and stationarity,
\[
\Gamma_\ell
=\E_\pi\!\left[\xi(Z_0)(P^\ell\xi)(Z_0)^\top\right].
\]
For a rank-one matrix $uv^\top$, the nuclear norm is $\norm{u}\norm{v}$. Lemma~\ref{lem:mixing-decay}, applied to the centered function $\xi$, therefore gives
\begin{align*}
\norm{\Gamma_\ell}_*
&\le\E_\pi\!\left[\norm{\xi(Z_0)}\norm{P^\ell\xi(Z_0)}\right]\\
&\le b\cdot 2C_Z\rho^\ell b
=2b^2C_Z\rho^\ell.
\end{align*}
The term $\Gamma_0$ is finite because $\xi$ is bounded. Since $\Gamma_{-\ell}=\Gamma_\ell^\top$ and the nuclear norm is invariant under transposition, the same geometric estimate holds for negative lags. Summing over $\ell\in\mathbb Z$ proves \eqref{eq:Cgamma}.
\end{proof}

\begin{proof}[Proof of Lemma~\ref{lem:poisson}]
The proof follows the two Poisson equations in the order in which they are used later.

\emph{Step 1: construct the noise corrector.}
Apply Lemma~\ref{lem:mixing-decay} to the centered function $\xi$. Since $P^j\xi(z)=\E[\xi(Z_j)\mid Z_0=z]$, the $j$th term represents the expected noise $j$ steps into the future. The series
\[
V(z)=\sum_{j=0}^{\infty}P^j\xi(z)
\]
converges uniformly because the $j$th term is bounded by $2bC_Z\rho^j$. In particular,
\begin{equation}
\norm{V}_\infty
\le \frac{2bC_Z}{1-\rho}
=:v_V.
\label{eq:v-bound}
\end{equation}
Uniform convergence permits termwise application of $P$, so
\[
V-PV
=\sum_{j=0}^{\infty}P^j\xi-\sum_{j=1}^{\infty}P^j\xi
=\xi.
\]

\emph{Step 2: obtain the martingale--coboundary decomposition.}
Define
\[
D_{k+1}=V(Z_{k+1})-PV(Z_k).
\]
The Markov property gives
\[
\E[V(Z_{k+1})\mid\cF_k]=PV(Z_k),
\]
and hence $\E[D_{k+1}\mid\cF_k]=0$. Also,
\[
D_{k+1}+V(Z_k)-V(Z_{k+1})
=V(Z_k)-PV(Z_k)
=\xi(Z_k).
\]
Thus the first identity in \eqref{eq:mart-cob} holds. Equation~\eqref{eq:v-bound} gives the uniform increment bound
\begin{equation}
\norm{D_{k+1}}\le2v_V
\qquad\text{almost surely}.
\label{eq:D-bound-app}
\end{equation}
A direct conditional second-moment calculation yields
\begin{align*}
\E[D_{k+1}D_{k+1}^\top\mid\cF_k]
&=P(VV^\top)(Z_k)-(PV)(Z_k)(PV)(Z_k)^\top\\
&=\cC(Z_k).
\end{align*}

\emph{Step 3: identify the mean conditional covariance.}
The absolute convergence in \eqref{eq:Cgamma} allows us to sum the lag covariances inside expectation. Since $V=\sum_{j\ge0}P^j\xi$, the Markov property gives
\[
\E_\pi[\xi(Z_0)V(Z_0)^\top]
=\sum_{j\ge0}\E_\pi[\xi(Z_0)(P^j\xi)(Z_0)^\top]
=\sum_{j\ge0}\Gamma_j,
\qquad
\E_\pi[V(Z_0)\xi(Z_0)^\top]
=\sum_{j\ge0}\Gamma_j^\top.
\]
The term $\Gamma_0$ appears in both sums. Therefore
\begin{equation}
\Sigma_M
=\E_\pi\!\left[
\xi V^\top+V\xi^\top-\xi\xi^\top
\right],
\end{equation}
where all functions in the expectation are evaluated at $Z_0$. Substitute $\xi=V-PV$ into the matrix inside the expectation. The cross terms cancel and give
\[
\xi V^\top+V\xi^\top-\xi\xi^\top
=VV^\top-(PV)(PV)^\top.
\]
Finally, invariance of $\pi$ implies
\[
\E_\pi[VV^\top]=\E_\pi[P(VV^\top)].
\]
Combining these identities with the definition of $\cC$ proves
\begin{equation}
\int\cC(z)\pi(\dd z)=\Sigma_M.
\label{eq:sigmam-C}
\end{equation}
This direct calculation is the reason the covariance in the Gaussian replacement is exactly the long-run covariance.

\emph{Step 4: construct the covariance corrector.}
The matrix $\cC(z)$ is positive semidefinite. By \eqref{eq:D-bound-app},
\[
\sup_z\norm{\cC(z)}_*
=\sup_z\tr(\cC(z))
\le4v_V^2.
\]
Thus $g:=\cC-\Sigma_M$ is a bounded matrix-valued function, and \eqref{eq:sigmam-C} shows that $\pi g=0$. Lemma~\ref{lem:mixing-decay}, applied to symmetric matrices with the nuclear norm, shows that
\[
W=\sum_{j=0}^{\infty}P^jg
\]
converges uniformly. Termwise subtraction gives $W-PW=g$, and
\[
\norm{W}_{\infty,*}
\le\frac{2C_Z}{1-\rho}\norm{\cC-\Sigma_M}_{\infty,*}
<\infty.
\]
This proves all claims.
\end{proof}

\section{Proof of the one-block Gaussian approximation}
\label{app:block-proof}
This appendix proves Proposition~\ref{prop:block-defect}. Fix $T>0$, restrict to $0<\alpha\le1$, let $n=\lceil T/\alpha\rceil$, and assume $(Y_0,Z_0)\sim\gamma\otimes\pi$. The comparison proceeds from the nonlinear recursion to its linearization, then to a weighted martingale sum, then to a Gaussian weighted sum, and finally to the target Gaussian law. Appendix~\ref{app:reader-guide} gives an informal derivation of the same path. Here every approximation is stated with the estimate used in the final triangle inequality.
\subsection{Linearization and finite-block moment bounds}

Set
\begin{equation}
A:=\Id_d-\alpha H,
\qquad
\bar Y_{k+1}=A\bar Y_k+\sqrt{\alpha}\,\xi(Z_k),
\qquad
\bar Y_0=Y_0.
\end{equation}
The matrix $A$ is the explicit Euler update for the linearized drift $-Hy$. To compare the two recursions, observe that
\[
\mathcal T_\alpha(y)-Ay
=-\sqrt{\alpha}\bigl[\nabla f(x_\star+\sqrt{\alpha}y)-H\sqrt{\alpha}y\bigr].
\]
The integral Taylor formula gives
\[
\nabla f(x_\star+\sqrt{\alpha}y)-H\sqrt{\alpha}y
=\int_0^1
\bigl[\nabla^2f(x_\star+t\sqrt{\alpha}y)-H\bigr]
\sqrt{\alpha}y\,\dd t.
\]
Using the Hessian Lipschitz condition and integrating $t$ over $[0,1]$ yields
\begin{equation}
\norm{\mathcal T_\alpha(y)-Ay}
\le\frac{M}{2}\alpha^{3/2}\norm{y}^2.
\label{eq:linearization-local}
\end{equation}

The linear recursion has the explicit representation
\[
\bar Y_k=A^kY_0+\sqrt{\alpha}\sum_{i=0}^{k-1}A^{k-1-i}\xi(Z_i).
\]
The two terms are centered and independent because $Y_0$ is independent of the stationary Markov chain. For $\alpha\le L^{-1}$, the symmetric matrix $A$ has spectrum in $[0,1]$, so $\norm{A^j}_{\op}\le1$. The next calculation shows that the weighted Markov sum has a second moment bounded uniformly over a fixed block. Using \eqref{eq:Gamma-def}, stationarity gives
\begin{align}
\E\norm{\bar Y_k}^2
&=\tr(A^k\Sigma A^k)
 +\alpha\sum_{i,j=0}^{k-1}
 \tr\!\left(A^{k-1-i}\Gamma_{j-i}A^{k-1-j}\right).
\end{align}
Stationarity expresses all cross terms through the lag covariances $\Gamma_{j-i}$, whose absolute summability controls their total contribution. The first term is at most $\tr(\Sigma)$. For the second term, cyclicity of the trace and operator--nuclear norm duality give
\[
\left|\tr\!\left(A^{k-1-i}\Gamma_{j-i}A^{k-1-j}\right)\right|
\le\norm{\Gamma_{j-i}}_*.
\]
For each $i$, the sum over $j$ is bounded by $C_\Gamma$. Therefore
\begin{equation}
\sup_{0\le k\le n}\E\norm{\bar Y_k}^2
\le \tr(\Sigma)+(T+1)C_\Gamma
=:C_{\mathrm{mom}}(T),
\label{eq:linear-second}
\end{equation}
where we used $\alpha n\le T+1$.

Construct $Y_k$ and $\bar Y_k$ with the same initial condition and the same Markov path. Add and subtract $\mathcal T_\alpha(\bar Y_k)$, use the contraction \eqref{eq:T-contract}, and apply \eqref{eq:linearization-local} at $\bar Y_k$:
\[
\norm{Y_{k+1}-\bar Y_{k+1}}
\le(1-m\alpha)\norm{Y_k-\bar Y_k}
+\frac{M}{2}\alpha^{3/2}\norm{\bar Y_k}^2.
\]
Iteration from $Y_0=\bar Y_0$ gives
\begin{align}
\E\norm{Y_n-\bar Y_n}
&\le\frac{M}{2}\alpha^{3/2}
\sum_{k=0}^{n-1}(1-m\alpha)^{n-1-k}\E\norm{\bar Y_k}^2\notag\\
&\le\frac{M C_{\mathrm{mom}}(T)}{2m}\sqrt{\alpha}.
\label{eq:nonlin-linear}
\end{align}
The final bound uses \eqref{eq:linear-second} and the geometric-series estimate
$\sum_{j\ge0}(1-m\alpha)^j=(m\alpha)^{-1}$. Thus the nonlinear-to-linear comparison contributes $O(\sqrt{\alpha})$ over one block.

\subsection{Weighted martingale--coboundary decomposition}

Insert \eqref{eq:mart-cob} into the explicit solution of the linear recursion. With the reindexing $r=i+1$, define
\begin{equation}
B_r=\sqrt{\alpha}\,A^{n-r},
\qquad r=1,\ldots,n.
\label{eq:B-def}
\end{equation}
The coboundary terms can be summed exactly. This is the matrix-weighted version of the elementary identity $\sum_r(v_{r-1}-v_r)=v_0-v_n$. Indeed,
\begin{align}
&\sum_{r=1}^nB_r\bigl[V(Z_{r-1})-V(Z_r)\bigr]\notag\\
&\quad=B_1V(Z_0)-B_nV(Z_n)
+\sum_{r=1}^{n-1}(B_{r+1}-B_r)V(Z_r).
\label{eq:weighted-summation-parts}
\end{align}
Here $B_1=\sqrt\alpha A^{n-1}$, $B_n=\sqrt\alpha\Id_d$, and
\[
B_{r+1}-B_r
=\alpha^{3/2}A^{n-r-1}H.
\]
Substitution into the linear recursion gives the exact decomposition
\begin{equation}
\bar Y_n=A^nY_0+\sum_{r=1}^nB_rD_r+R_{V,n},
\label{eq:L-decomp}
\end{equation}
where
\begin{equation}
R_{V,n}=\sqrt{\alpha}\left[
A^{n-1}V(Z_0)-V(Z_n)
+\alpha\sum_{j=1}^{n-1}A^{n-1-j}HV(Z_j)
\right].
\end{equation}
The first two terms are the endpoints of the coboundary sum. The final term is the accumulated variation of the matrix weights in \eqref{eq:weighted-summation-parts}. Since $\norm{A^j}_{\op}\le1$, $\norm{H}_{\op}\le L$, and $\alpha n\le T+1$, we obtain
\begin{equation}
\norm{R_{V,n}}
\le v_V\bigl(2+L(T+1)\bigr)\sqrt{\alpha}.
\label{eq:cob-bound}
\end{equation}
Consequently, replacing $\bar Y_n$ by the weighted martingale sum in \eqref{eq:L-decomp} costs $O(\sqrt{\alpha})$ in Wasserstein distance.

\subsection{Gaussian smoothing for Lipschitz test functions}

Let $N_1,\ldots,N_n$ be independent $\cN(0,\Sigma_M)$ vectors, independent of $Y_0$ and the Markov chain. Set $\beta_1=\E\norm{N_1}$ and $\beta_3=\E\norm{N_1}^3$. Both quantities are finite and depend only on $\Sigma_M$. We compare
\[
A^nY_0+\sum_{r=1}^nB_rD_r
\qquad\text{with}\qquad
A^nY_0+\sum_{r=1}^nB_rN_r.
\]
The propagated initial condition provides the smoothing needed for this comparison. Since $Y_0\sim\cN(0,\Sigma)$, the covariance of $A^nY_0$ is $A^n\Sigma A^n$. Restrict $\alpha\le(2L)^{-1}$. Then $\lambda_{\min}(A)\ge1-\alpha L$ and $\log(1-\alpha L)\ge-2\alpha L$. Since $\alpha n\le T+1$,
\begin{equation}
A^n\Sigma A^n
\succeq \lambda_{\min}(\Sigma)e^{-4L(T+1)}\Id_d
=:c_T\Id_d.
\end{equation}
Thus the smoothing covariance remains uniformly nondegenerate over the block. For a $1$-Lipschitz test function $h$, define
\[
\psi(u)=\E[h(A^nY_0+u)].
\]
Gaussian convolution makes $\psi$ twice differentiable with a Lipschitz Hessian. Appendix~\ref{app:smoothing} proves the uniform bounds
\begin{equation}
\sup_u\norm{\nabla^2\psi(u)}_{\op}
\le a_{2,T},
\qquad
\norm{\nabla^2\psi(u)-\nabla^2\psi(v)}_{\op}
\le a_{3,T}\norm{u-v},
\label{eq:smoothing-bounds}
\end{equation}
where one may take $a_{2,T}=a_2(c_T)$ and $a_{3,T}=a_3(c_T)$ with the functions defined in Appendix~\ref{app:smoothing}. These constants are independent of $\alpha$.

\subsection{Lindeberg replacement}
For $r=1,\ldots,n$, define the Lindeberg hybrid
\begin{equation}
U_r=\sum_{i=1}^{r-1}B_iD_i+\sum_{i=r+1}^{n}B_iN_i.
\end{equation}
Thus the increments before position $r$ are martingale increments, those after position $r$ are Gaussian, and the $r$th position is left open. Empty sums are zero. Replacing the open position and summing over $r$ gives the Lindeberg telescoping identity \citep{Chatterjee2006}
\begin{align}
&\E\psi\!\left(\sum_{i=1}^nB_iD_i\right)
-\E\psi\!\left(\sum_{i=1}^nB_iN_i\right)\notag\\
&\quad=\sum_{r=1}^n\left[
\E\psi(U_r+B_rD_r)-\E\psi(U_r+B_rN_r)
\right].
\end{align}
Indeed, $U_r+B_rD_r$ is the hybrid with martingale increments through position $r$, while $U_r+B_rN_r$ is the hybrid with martingale increments only through position $r-1$. Consecutive hybrids therefore cancel in the sum.
For the $r$th summand, condition on
\[
\mathcal G_r=\cF_{r-1}\vee\sigma(N_i:r<i\le n).
\]
This $\sigma$-field contains every random variable already present in $U_r$, so $U_r$ is $\mathcal G_r$-measurable. Adding the independent future Gaussian variables does not change conditional moments of $D_r$. The martingale property therefore gives
\[
\E[D_r\mid\mathcal G_r]=0,
\qquad
\E[D_rD_r^\top\mid\mathcal G_r]=\cC(Z_{r-1}),
\]
while $N_r$ is conditionally centered with covariance $\Sigma_M$.

Apply the second-order Taylor formula for $\psi$ at $U_r$. For any increment $x\in\R^d$, the Hessian-Lipschitz bound in \eqref{eq:smoothing-bounds} gives
\begin{equation}
\left|
\psi(U_r+x)-\psi(U_r)-\ip{\nabla\psi(U_r)}{x}
-\frac12x^\top\nabla^2\psi(U_r)x
\right|
\le\frac{a_{3,T}}{6}\norm{x}^3.
\label{eq:taylor-lindeberg}
\end{equation}
Use \eqref{eq:taylor-lindeberg} first with $x=B_rD_r$ and then with $x=B_rN_r$. Conditional on $\mathcal G_r$, both increments are centered, so the linear terms vanish. Writing
\[
K_r=B_r^\top\nabla^2\psi(U_r)B_r,
\]
the difference of the quadratic terms is
\begin{align}
&\E\psi(U_r+B_rD_r)-\E\psi(U_r+B_rN_r)\notag\\
&\quad=\frac12\E\frobip{K_r}{\cC(Z_{r-1})-\Sigma_M}+\mathcal R_r.
\label{eq:lindeberg-one}
\end{align}
By \eqref{eq:D-def} and boundedness of $V$, $\norm{D_r}\le2v_V$, while \eqref{eq:B-def} gives $\norm{B_r}_{\op}\le\sqrt\alpha$. The two Taylor remainders therefore satisfy the explicit bound
\begin{equation}
|\mathcal R_r|
\le\frac{a_{3,T}}{6}\alpha^{3/2}
\left((2v_V)^3+\beta_3\right).
\label{eq:lindeberg-remainder-one}
\end{equation}
Since $n\le T/\alpha+1$, summing \eqref{eq:lindeberg-remainder-one} gives
\begin{equation}
\sum_{r=1}^n|\mathcal R_r|
\le C_T n\alpha^{3/2}
\le C_T\sqrt{\alpha}.
\label{eq:lindeberg-third}
\end{equation}

\subsection{Predictable covariance correction}
We next sum the predictable covariance errors. By \eqref{eq:mart-cob},
\[
\cC(Z_{r-1})-\Sigma_M=W(Z_{r-1})-PW(Z_{r-1}).
\]
The key point is that $K_r$ is $\mathcal G_r$-measurable even though it remains correlated with $Z_{r-1}$. Thus $K_r$ is predictable for the transition from $Z_{r-1}$ to $Z_r$. The Markov property and conditional expectation give
\begin{align*}
\E\frobip{K_r}{W(Z_r)}
&=\E\!\left[\E\!\left[\frobip{K_r}{W(Z_r)}\mid\mathcal G_r\right]\right]\\
&=\E\frobip{K_r}{PW(Z_{r-1})}.
\end{align*}
Thus the correlation between $K_r$ and the current Markov state is retained inside an exact telescoping identity. More precisely, each summand equals
\[
\E\frobip{K_r}{W(Z_{r-1})}
-\E\frobip{K_r}{W(Z_r)}.
\]
After summing over $r$, the first sum is indexed by $Z_0,\ldots,Z_{n-1}$ and the second by $Z_1,\ldots,Z_n$. Pairing the common interior states gives
\begin{align}
\sum_{r=1}^n\E\frobip{K_r}{\cC(Z_{r-1})-\Sigma_M}
&=\E\frobip{K_1}{W(Z_0)}-\E\frobip{K_n}{W(Z_n)}\notag\\
&\quad+\sum_{r=1}^{n-1}\E\frobip{K_{r+1}-K_r}{W(Z_r)}.
\label{eq:cov-telescope}
\end{align}
The boundary terms can be bounded directly. By \eqref{eq:smoothing-bounds} and \eqref{eq:B-def},
\[
\norm{K_r}_{\op}
\le a_{2,T}\norm{B_r}_{\op}^2
\le a_{2,T}\alpha.
\]
Operator--nuclear norm duality therefore gives
\begin{equation}
\left|\E\frobip{K_1}{W(Z_0)}\right|
+\left|\E\frobip{K_n}{W(Z_n)}\right|
\le2a_{2,T}\norm{W}_{\infty,*}\alpha.
\end{equation}

For the variation terms, the weights and hybrids satisfy, for $r=1,\ldots,n-1$,
\begin{equation}
\norm{B_{r+1}-B_r}_{\op}\le L\alpha^{3/2},
\qquad
U_{r+1}-U_r=B_rD_r-B_{r+1}N_{r+1}.
\label{eq:weight-hybrid-increments}
\end{equation}
Write $J_r=\nabla^2\psi(U_r)$. Adding and subtracting $B_r^\top J_{r+1}B_{r+1}$ and $B_r^\top J_rB_{r+1}$ gives
\begin{align*}
K_{r+1}-K_r
&=(B_{r+1}-B_r)^\top J_{r+1}B_{r+1}\\
&\quad+B_r^\top(J_{r+1}-J_r)B_{r+1}\\
&\quad+B_r^\top J_r(B_{r+1}-B_r).
\end{align*}
The first and third terms measure the change in the weights. The middle term measures the change in the smoothed Hessian along the hybrid increment. Using \eqref{eq:smoothing-bounds} and \eqref{eq:weight-hybrid-increments},
\begin{align*}
\E\norm{K_{r+1}-K_r}_{\op}
&\le 2a_{2,T}L\alpha^2
 +a_{3,T}\alpha\E\norm{U_{r+1}-U_r}\\
&\le 2a_{2,T}L\alpha^2
 +a_{3,T}(2v_V+\beta_1)\alpha^{3/2}.
\end{align*}
Summing over $r=1,\ldots,n-1$ gives
\begin{equation}
\sum_{r=1}^{n-1}\E\norm{K_{r+1}-K_r}_{\op}
\le C_T\sqrt{\alpha}.
\end{equation}
Multiplying the variation bound by $\norm{W}_{\infty,*}$ and combining it with the two boundary terms shows that the right-hand side of \eqref{eq:cov-telescope} is $O(\sqrt{\alpha})$. Summing \eqref{eq:lindeberg-one}, using \eqref{eq:lindeberg-third}, and taking the supremum over all $1$-Lipschitz $h$ yields
\begin{equation}
\Wone\!\left(
\Law\!\left(A^nY_0+\sum_{r=1}^nB_rD_r\right),
\Law\!\left(A^nY_0+\sum_{r=1}^nB_rN_r\right)
\right)
\le C_T\sqrt{\alpha}.
\label{eq:mart-gauss-main}
\end{equation}

\subsection{Identification of the target covariance}
The Gaussian vector in the second law has covariance
\begin{equation}
\Sigma_{\alpha,n}
=A^n\Sigma A^n
+\alpha\sum_{j=0}^{n-1}A^j\Sigma_MA^j.
\end{equation}
Using $\Sigma_M=H\Sigma+\Sigma H$ and $A=\Id_d-\alpha H$ gives
\[
A\Sigma A+\alpha\Sigma_M
=\Sigma+\alpha^2H\Sigma H.
\]
Rearrange this identity as
$A\Sigma A-\Sigma=-\alpha\Sigma_M+\alpha^2H\Sigma H$.
Multiplying on both sides by $A^j$ and summing over $j=0,\ldots,n-1$ makes the covariance terms telescope and gives
\begin{equation}
\Sigma_{\alpha,n}-\Sigma
=\alpha^2\sum_{j=0}^{n-1}A^jH\Sigma H A^j.
\label{eq:cov-identity}
\end{equation}
Since $\norm{A^j}_{\op}\le1$ and $n\alpha\le T+1$,
\[
\norm{\Sigma_{\alpha,n}-\Sigma}_F
\le (T+1)\alpha\norm{H\Sigma H}_F.
\]
Moreover, \eqref{eq:cov-identity} shows that $\Sigma_{\alpha,n}\succeq\Sigma$. Hence both covariance matrices are bounded below by $\lambda_{\min}(\Sigma)\Id_d$. Coupling the two Gaussian vectors through the same standard normal vector and applying Lemma~\ref{lem:gaussian-cov} give
\begin{equation}
\Wone\bigl(\cN(0,\Sigma_{\alpha,n}),\cN(0,\Sigma)\bigr)
\le C_T\alpha.
\label{eq:gauss-cov-bound}
\end{equation}

Finally, apply the triangle inequality along the comparison chain displayed at the beginning of the section. The nonlinear-to-linear term is bounded by \eqref{eq:nonlin-linear}, the coboundary term by \eqref{eq:cob-bound}, the martingale-to-Gaussian term by \eqref{eq:mart-gauss-main}, and the covariance term by \eqref{eq:gauss-cov-bound}. Therefore
\begin{align*}
\Wone(\Law(Y_n),\gamma)
&\le C_T\sqrt\alpha+C_T\alpha\\
&\le C'_T\sqrt\alpha,
\end{align*}
where the last inequality uses $\alpha\le1$. This proves Proposition~\ref{prop:block-defect}.

\section{Joint contraction and construction of the invariant law}
\label{app:joint-proof}

This section proves Proposition~\ref{prop:block-contract}, constructs the invariant law, and completes the proof of Theorem~\ref{thm:main}. The coupling first brings the two Markov paths together. Strong convexity then contracts the iterate coordinates because both recursions receive the same noise after coalescence.

\subsection{Contraction of the joint block kernel}

Fix two initial states $(y,z)$ and $(y',z')$. Lemma~\ref{lem:uniform-coupling} supplies a coupling of the two $P$-chains with coalescence time $\tau$. Write $\Pp_{z,z'}$ and $\E_{z,z'}$ for probability and expectation under this coupling when it starts from $(z,z')$. Expectations in this subsection are taken under that coupling. Then
\[
\Pp_{z,z'}(\tau>k)
\le C_{\mathrm{cpl}}\rho_{\mathrm{cpl}}^k
\ind_{\{z\ne z'\}}.
\]
Use the coupled Markov paths to drive the two scaled recursions. Set
\[
a_\alpha=1-m\alpha,
\qquad
\Delta_k=\norm{Y_k-Y'_k}.
\]
Before coalescence, the two scaled noise increments differ by at most $2b\sqrt{\alpha}$. After coalescence, the increments agree. Combining this observation with \eqref{eq:T-contract} gives the pathwise inequality
\begin{equation}
\Delta_{k+1}
\le a_\alpha\Delta_k
+2b\sqrt{\alpha}\ind_{\{\tau>k\}}.
\label{eq:Delta-rec}
\end{equation}
Iterating \eqref{eq:Delta-rec} and taking expectations yields
\begin{align}
\E\Delta_n
&\le a_\alpha^n\Delta_0
+2bC_{\mathrm{cpl}}\sqrt{\alpha}
\sum_{k=0}^{n-1}a_\alpha^{n-1-k}\rho_{\mathrm{cpl}}^k
\ind_{\{z\ne z'\}}.\notag
\end{align}
Since $a_\alpha\to1$ and $\rho_{\mathrm{cpl}}<1$, restrict $\alpha$ further so that
$a_\alpha-\rho_{\mathrm{cpl}}\ge(1-\rho_{\mathrm{cpl}})/2$. The geometric convolution can then be evaluated exactly:
\[
\sum_{k=0}^{n-1}a_\alpha^{n-1-k}\rho_{\mathrm{cpl}}^k
=\frac{a_\alpha^n-\rho_{\mathrm{cpl}}^n}
{a_\alpha-\rho_{\mathrm{cpl}}}
\le\frac{2a_\alpha^n}{1-\rho_{\mathrm{cpl}}}.
\]
Consequently,
\begin{equation}
\E\Delta_n
\le a_\alpha^n\Delta_0
+\frac{4bC_{\mathrm{cpl}}}{1-\rho_{\mathrm{cpl}}}
 a_\alpha^n\sqrt{\alpha}\ind_{\{z\ne z'\}}.
\label{eq:Delta-block}
\end{equation}
The same coupling controls the terminal Markov-state discrepancy:
\begin{equation}
\Pp(Z_n\ne Z'_n)
\le C_{\mathrm{cpl}}\rho_{\mathrm{cpl}}^n
\ind_{\{z\ne z'\}}.
\label{eq:Z-block}
\end{equation}

Choose $T>0$ so that $e^{-mT}\le1/4$ and take $n=\lceil T/\alpha\rceil$. Since $n\alpha\ge T$,
\[
a_\alpha^n\le e^{-m\alpha n}\le e^{-mT}\le\frac14.
\]
Next choose
\[
\kappa\ge\frac{4bC_{\mathrm{cpl}}}{1-\rho_{\mathrm{cpl}}},
\]
and restrict $\alpha$ so that
$C_{\mathrm{cpl}}\rho_{\mathrm{cpl}}^{\lceil T/\alpha\rceil}\le1/4$.
The initial iterate discrepancy is then multiplied by at most $1/4$. If the initial Markov states differ, the noise-mismatch term contributes at most $\kappa\sqrt\alpha/4$, and the terminal state mismatch contributes at most another $\kappa\sqrt\alpha/4$. Therefore \eqref{eq:Delta-block} and \eqref{eq:Z-block} give
\begin{equation}
\E d_{\alpha,\kappa}((Y_n,Z_n),(Y'_n,Z'_n))
\le\frac12d_{\alpha,\kappa}((Y_0,Z_0),(Y'_0,Z'_0)).
\label{eq:point-contract}
\end{equation}
To pass from point states to probability measures, let $h$ be an admissible joint Kantorovich--Rubinstein test. Coupling the two copies from point states $p,p'$ and using \eqref{eq:point-contract} gives
\[
|(\Palpha^n h)(p)-(\Palpha^n h)(p')|
\le\frac12d_{\alpha,\kappa}(p,p').
\]
For a fixed base point $p_0$, the function
$2[\Palpha^n h-(\Palpha^n h)(p_0)]$ is therefore admissible. The subtraction of the constant does not change its integral against $\lambda-\lambda'$. Integrating and taking the supremum over $h$ proves Proposition~\ref{prop:block-contract} with $q=1/2$.

\subsection{Construction and uniqueness of the invariant law}

Fix $0<\alpha\le L^{-1}$. Because $\Zspace$ is Polish and $\pi$ is invariant, a stationary $P$-chain with marginal $\pi$ has a two-sided realization $(Z_k)_{k\in\mathbb Z}$. Since $\mathcal T_\alpha(0)=0$, the contraction bound gives
\begin{equation}
\norm{\mathcal T_\alpha(y)+\sqrt{\alpha}\xi(z)}
\le(1-m\alpha)\norm{y}+b\sqrt{\alpha}.
\label{eq:absorbing-main}
\end{equation}
For each $N\ge1$, initialize the scaled recursion at time $-N$ from zero and denote its value at time zero by $Y_0^{(N)}$. Iterating \eqref{eq:absorbing-main} shows that every such trajectory stays in the deterministic ball of radius $b/(m\sqrt{\alpha})$.

If $N'>N$, the two trajectories use the same noise path from time $-N$ onward. At time $-N$, the trajectory started there equals zero, while the older trajectory lies in the absorbing ball. Their distance at that time is at most $b/(m\sqrt\alpha)$. Contractivity over the next $N$ updates gives
\[
\norm{Y_0^{(N')}-Y_0^{(N)}}
\le\frac{b}{m\sqrt{\alpha}}(1-m\alpha)^N.
\]
The right-hand side tends to zero, so $(Y_0^{(N)})$ is almost surely Cauchy. Let $Y_0$ denote its limit. This limit is measurable with respect to the infinite past of the noise chain.

Apply the same pullback construction after shifting the time origin. This defines $Y_k$ for every $k\in\mathbb Z$, and the shift invariance of the two-sided noise chain makes $(Y_k,Z_k)_{k\in\mathbb Z}$ stationary. Passing to the limit in the recursion shows that it satisfies \eqref{eq:scaled-rec} \citep{DiaconisFreedman1999,GuptaHaskell2021}. Hence $\varpi_\alpha:=\Law(Y_0,Z_0)$ is invariant for $\Palpha$ and has $Z$-marginal $\pi$.

The pullback law just constructed is supported in the deterministic ball
\begin{equation}
\norm{Y}\le\frac{b}{m\sqrt{\alpha}}
\qquad\text{almost surely}.
\label{eq:stationary-support-main}
\end{equation}
Indeed, the finite pullback trajectories stay in this ball by
\eqref{eq:absorbing-main}, and the ball is closed, so their almost-sure
limit has the same bound.

Now restrict further to $\alpha\le\alpha_c$, where $\alpha_c$ is supplied by Proposition~\ref{prop:block-contract}. Every invariant law has a finite first $Y$-moment. Indeed, iterating \eqref{eq:absorbing-main} gives
\begin{equation}
\norm{Y_k}
\le(1-m\alpha)^k\norm{Y_0}
+\frac{b}{m\sqrt\alpha}.
\end{equation}
If the law of $Y_0$ is invariant, then $Y_k$ has the same law as $Y_0$. For every $\varepsilon>0$, the preceding pathwise bound gives
\begin{align*}
\Pp\!\left(\norm{Y_0}>\frac{b}{m\sqrt\alpha}+\varepsilon\right)
&=\Pp\!\left(\norm{Y_k}>\frac{b}{m\sqrt\alpha}+\varepsilon\right)\\
&\le
\Pp\!\left(
\norm{Y_0}>
\varepsilon(1-m\alpha)^{-k}
\right).
\end{align*}
The threshold on the right tends to infinity, so the probability tends to zero. Thus every invariant law is supported in the same deterministic ball as \eqref{eq:stationary-support-main}. Its $Z$-marginal is $\pi$ because uniform ergodicity makes $\pi$ the unique invariant law of $P$.

Let $\mu_1$ and $\mu_2$ be two invariant laws of the joint recursion. Both have finite first $Y$-moments, so Proposition~\ref{prop:block-contract} applies:
\[
\Wjoint(\mu_1,\mu_2)
=\Wjoint(\mu_1\Palpha^n,\mu_2\Palpha^n)
\le q\Wjoint(\mu_1,\mu_2).
\]
The dual class contains suitably rescaled products of bounded Lipschitz functions of $y$ and bounded measurable functions of $z$, so it determines the joint law. Hence $q<1$ forces $\mu_1=\mu_2$, and the pullback law $\varpi_\alpha$ is the unique invariant law of the scaled recursion.

\subsection{Transfer of the block estimate to stationarity}

Recall $\nu=\gamma\otimes\pi$. Let $h$ be an admissible joint test and fix an arbitrary $z_0\in\Zspace$. The function $y\mapsto h(y,z_0)$ is $1$-Lipschitz. Moreover,
\[
|h(y,z)-h(y,z_0)|\le\kappa\sqrt\alpha
\qquad\text{for every }y,z.
\]
Both $\nu\Palpha^n$ and $\nu$ have $Z$-marginal $\pi$. Replacing $h(y,z)$ by $h(y,z_0)$ in each of the two expectations therefore costs at most $\kappa\sqrt\alpha$. Proposition~\ref{prop:block-defect} now gives
\begin{equation}
\Wjoint(\nu\Palpha^n,\nu)
\le(C_T+2\kappa)\sqrt\alpha.
\end{equation}

Set $w_\alpha=\Wjoint(\varpi_\alpha,\nu)$. Invariance of $\varpi_\alpha$, the triangle inequality, and Proposition~\ref{prop:block-contract} yield
\begin{align*}
w_\alpha
&=\Wjoint(\varpi_\alpha\Palpha^n,\nu)\\
&\le
\Wjoint(\varpi_\alpha\Palpha^n,\nu\Palpha^n)
+\Wjoint(\nu\Palpha^n,\nu)\\
&\le qw_\alpha+(C_T+2\kappa)\sqrt\alpha.
\end{align*}
Therefore
\begin{equation}
w_\alpha
\le\frac{C_T+2\kappa}{1-q}\sqrt\alpha.
\end{equation}
Projection onto the first coordinate can only decrease Wasserstein distance. Hence the $Y$-marginal of $\varpi_\alpha$ is within $C\sqrt\alpha$ of $\gamma$. The pushforward of $\varpi_\alpha$ under $(y,z)\mapsto(x_\star+\sqrt\alpha y,z)$ is the unique invariant probability measure $\Pi_\alpha$ of the original recursion. Scaling the iterate distance by $\sqrt\alpha$ also gives the unscaled $O(\alpha)$ bound. Taking $\alpha_\star$ to be the minimum of the finitely many positive thresholds introduced above completes the proof of Theorem~\ref{thm:main}.

\section{Gaussian smoothing and covariance perturbation}
\label{app:smoothing}

This appendix records the two analytic estimates used in the Lindeberg argument. Gaussian convolution turns a Lipschitz test function into a function of class $C^2$ with a quantitatively Lipschitz Hessian. A separate matrix perturbation estimate converts a covariance error into a Wasserstein-1 error between centered Gaussian laws.

\begin{lemma}[Derivative bounds for Gaussian convolution]
Let $S\succeq c\Id_d$ with $c>0$, let $G\sim\cN(0,S)$, and let $h:\R^d\to\R$ be $1$-Lipschitz. Define $\psi(x)=\E[h(x+G)]$. Then $\psi$ is of class $C^2$ and
\begin{align}
\sup_x\norm{\nabla^2\psi(x)}_{\op}
&\le a_2(c):=\sqrt{\frac2\pi}\,c^{-1/2},
\label{eq:a2}\\
\norm{\nabla^2\psi(x)-\nabla^2\psi(y)}_{\op}
&\le a_3(c)\norm{x-y},
\qquad a_3(c):=\sqrt2\,c^{-1}.
\label{eq:a3}
\end{align}
Consequently, for all $x,u\in\R^d$,
\begin{equation}
\left|\psi(x+u)-\psi(x)-\ip{\nabla\psi(x)}{u}
-\frac12u^\top\nabla^2\psi(x)u\right|
\le\frac{a_3(c)}6\norm{u}^3.
\label{eq:taylor-psi}
\end{equation}
\end{lemma}

\begin{proof}
For a $1$-Lipschitz function $h$, begin with a standard mollification $h_\delta$. Mollification preserves the Lipschitz constant, so $\norm{\nabla h_\delta}_\infty\le1$, and the estimates below are uniform in $\delta$. After subtracting the constant $h(0)$, these functions have at most linear growth. Dominated convergence in the Gaussian formulas permits $\delta\downarrow0$. It therefore suffices to prove the estimates for smooth $h$ with $\norm{\nabla h}_\infty\le1$.

Fix unit vectors $u,v\in\R^d$. Differentiating in direction $v$ and applying Gaussian integration by parts in direction $u$ gives
\begin{equation}
u^\top\nabla^2\psi(x)v
=\E\!\left[\partial_vh(x+G)\,u^\top S^{-1}G\right].
\end{equation}
The scalar $u^\top S^{-1}G$ is centered Gaussian with variance $u^\top S^{-1}u\le c^{-1}$. Since $|\partial_vh|\le1$,
\[
|u^\top\nabla^2\psi(x)v|
\le\E|u^\top S^{-1}G|
\le\sqrt{\frac2\pi}\,c^{-1/2}.
\]
Taking the supremum over unit $u$ and $v$ proves \eqref{eq:a2}.

A second Gaussian integration by parts gives
\begin{equation}
u^\top\nabla^2\psi(x)v
=\E[h(x+G)Q_{u,v}(G)],
\label{eq:ibp2}
\end{equation}
where
\[
Q_{u,v}(G)
=(u^\top S^{-1}G)(v^\top S^{-1}G)-u^\top S^{-1}v.
\]
The random variable $Q_{u,v}(G)$ is centered. Subtracting \eqref{eq:ibp2} at $x$ and $y$, then using the Lipschitz property of $h$, yields
\begin{align*}
|u^\top(\nabla^2\psi(x)-\nabla^2\psi(y))v|
&\le\norm{x-y}\E|Q_{u,v}(G)|\\
&\le\norm{x-y}(\E Q_{u,v}(G)^2)^{1/2}.
\end{align*}
Write $G=S^{1/2}G_0$ with $G_0\sim\cN(0,\Id_d)$. The Gaussian fourth-moment identity gives
\[
\E Q_{u,v}(G)^2
=(u^\top S^{-1}u)(v^\top S^{-1}v)
 +(u^\top S^{-1}v)^2
\le2c^{-2}.
\]
Taking the supremum over unit $u,v$ proves \eqref{eq:a3}.

Finally, for arbitrary $x,w\in\R^d$, Taylor's formula with integral remainder gives
\begin{align*}
&\psi(x+w)-\psi(x)-\ip{\nabla\psi(x)}{w}
-\frac12w^\top\nabla^2\psi(x)w\\
&\quad=\int_0^1(1-t)
 w^\top[\nabla^2\psi(x+tw)-\nabla^2\psi(x)]w\,\dd t.
\end{align*}
Using \eqref{eq:a3} and $\int_0^1t(1-t)\,\dd t=1/6$ proves \eqref{eq:taylor-psi}.
\end{proof}

\begin{lemma}[Perturbation of centered Gaussian laws]
\label{lem:gaussian-cov}
Let $S_1,S_2\succ0$ and suppose $S_1,S_2\succeq\lambda\Id_d$ for some $\lambda>0$. Then
\begin{equation}
\Wone(\cN(0,S_1),\cN(0,S_2))
\le\frac{\norm{S_1-S_2}_F}{2\sqrt{\lambda}}.
\end{equation}
\end{lemma}

\begin{proof}
Let $G_0\sim\cN(0,\Id_d)$ and use the common-randomness coupling $S_1^{1/2}G_0$ and $S_2^{1/2}G_0$. Jensen's inequality gives
\begin{align*}
\Wone(\cN(0,S_1),\cN(0,S_2))
&\le\E\norm{(S_1^{1/2}-S_2^{1/2})G_0}\\
&\le\bigl(\E\norm{(S_1^{1/2}-S_2^{1/2})G_0}^2\bigr)^{1/2}\\
&=\norm{S_1^{1/2}-S_2^{1/2}}_F.
\end{align*}
We next control the difference of the matrix square roots. Set
$\Delta_S=S_1^{1/2}-S_2^{1/2}$. From this definition, we obtain the Sylvester identity
\[
S_1-S_2=S_1^{1/2}\Delta_S+\Delta_S S_2^{1/2}.
\]
Cauchy--Schwarz for the Frobenius inner product then gives
\[
\norm{S_1-S_2}_F\norm{\Delta_S}_F
\ge\left|\frobip{\Delta_S}{S_1-S_2}\right|
=\frobip{\Delta_S}{S_1^{1/2}\Delta_S+\Delta_S S_2^{1/2}}.
\]
The two terms on the right are nonnegative. More explicitly,
\begin{align*}
\frobip{\Delta_S}{S_1^{1/2}\Delta_S}
&=\tr(\Delta_S^\top S_1^{1/2}\Delta_S)
\ge\sqrt\lambda\norm{\Delta_S}_F^2,\\
\frobip{\Delta_S}{\Delta_S S_2^{1/2}}
&=\tr(\Delta_S^\top\Delta_S S_2^{1/2})
\ge\sqrt\lambda\norm{\Delta_S}_F^2.
\end{align*}
Here we used $S_i^{1/2}\succeq\sqrt\lambda\Id_d$ and the fact that $\Delta_S^\top\Delta_S$ is positive semidefinite. Therefore
\[
\frobip{\Delta_S}{S_1^{1/2}\Delta_S+\Delta_S S_2^{1/2}}
\ge2\sqrt\lambda\norm{\Delta_S}_F^2.
\]
For $\Delta_S\ne0$, division by $2\sqrt\lambda\norm{\Delta_S}_F$ yields
$\norm{\Delta_S}_F\le\norm{S_1-S_2}_F/(2\sqrt\lambda)$. The same inequality is immediate when $\Delta_S=0$. Combining this estimate with the coupling bound proves the result. Related matrix square-root perturbation inequalities are discussed by \citet{Schmitt1992,Bhatia2007}.
\end{proof}

\section{Matching lower bound from temporal dependence}
\label{app:lower}

This appendix proves Theorem~\ref{thm:lower}. The calculation has four steps. We first verify that the finite-state noise has a symmetric marginal distribution and zero autocovariance at every nonzero lag. We then identify the third-order temporal moment that remains and compute the exact second and third moments of the stationary weighted series. Kantorovich--Rubinstein duality requires a Lipschitz test function, whereas $x\mapsto x^3$ is not Lipschitz. We therefore evaluate the bounded $1$-Lipschitz test function $\sin x$ through an explicit two-state characteristic-function recursion.

\subsection{Markov model and third-order temporal dependence}

Take $f(x)=x^2/2$. Then $x_\star=0$, $H=1$, and the objective assumptions in \eqref{eq:f-ass} hold with $m=L=1$ and $M=0$. Let $(\eta_k)_{k\in\mathbb Z}$ be independent Rademacher random variables, so
\begin{equation}
\Pp(\eta_k=1)=\Pp(\eta_k=-1)=\frac12.
\end{equation}
Define
\begin{equation}
Z_k=(\eta_{k-1},\eta_k),
\qquad
\xi(Z_k)=\frac{\eta_k(3-\eta_{k-1})}{2}.
\end{equation}
The four possible noise values are displayed by
\[
\begin{array}{c|cc}
&\eta_k=1&\eta_k=-1\\
\hline
\eta_{k-1}=1&1&-1\\
\eta_{k-1}=-1&2&-2
\end{array}
\]
The state space is $\{-1,1\}^2$. From a state $(a,b)$, the chain moves to $(b,c)$, where $c$ is a fresh Rademacher variable. Thus
\[
P((a,b),\{(b,c)\})=\frac12,
\qquad c\in\{-1,1\}.
\]
Its invariant probability measure $\pi$ is uniform on the four states. After two transitions, both coordinates are fresh and independent of the initial state, which gives
\begin{equation}
P^2(z,\cdot)=\pi
\qquad\text{for every }z.
\label{eq:lower-two-step-mixing}
\end{equation}
Equation \eqref{eq:lower-two-step-mixing} verifies \eqref{eq:uniform-ergodicity}, for example with $C_Z=1$ and $\rho=1/2$. The observable satisfies \eqref{eq:xi-ass} with $b=2$.

The four equally likely values of $\xi(Z_0)$ are $-2,-1,1,2$. Hence the marginal distribution is symmetric and
\begin{equation}
\E\xi(Z_0)=0,
\qquad
\E\xi(Z_0)^2=\frac52,
\qquad
\E\xi(Z_0)^3=0.
\label{eq:lower-marginal-moments}
\end{equation}
For every $\ell\ge1$, the variable $\xi(Z_\ell)$ contains the fresh factor $\eta_\ell$ to the first power. Conditional on all earlier Rademacher variables, this factor has mean zero. Therefore
\begin{equation}
\E[\xi(Z_0)\xi(Z_\ell)]=0,
\qquad \ell\ge1.
\label{eq:lower-cov-zero-app}
\end{equation}
Stationarity gives the same conclusion at negative lags. The long-run covariance is therefore
\begin{equation}
\Sigma_M=\E\xi(Z_0)^2=\frac52.
\end{equation}
Since $H=1$, the Lyapunov equation $2\Sigma=\Sigma_M$ gives $\Sigma=5/4$.

At second order, this process agrees with independent draws having the same marginal distribution because every nonzero-lag covariance vanishes. Third-order moments separate the two models. Direct conditioning gives
\begin{align}
\E[\xi(Z_0)\xi(Z_1)^2]
&=\E\left[
\frac{\eta_0(3-\eta_{-1})}{2}
\frac{(3-\eta_0)^2}{4}
\right]\notag\\
&=\frac18\E[3-\eta_{-1}]\,
\E[\eta_0(3-\eta_0)^2]
=-\frac94.
\label{eq:lower-mixed-app}
\end{align}
Here $\E[3-\eta_{-1}]=3$ and
$\E[\eta_0(3-\eta_0)^2]=-6$. By contrast, for independent centered variables $\widetilde\xi_0$ and $\widetilde\xi_1$ with this same marginal law,
\[
\E[\widetilde\xi_0\widetilde\xi_1^2]
=\E\widetilde\xi_0\,\E\widetilde\xi_1^2=0.
\]
Thus \eqref{eq:lower-mixed-app} is generated by temporal dependence.

\subsection{Stationary weighted series and exact low-order moments}

Fix $\alpha\in(0,1)$ and set $r_\alpha=1-\alpha$. The recursion is
\[
X_{k+1}=r_\alpha X_k+\alpha\xi(Z_k).
\]
Iterating backward $N$ steps in a two-sided stationary realization gives
\[
X_0=r_\alpha^N X_{-N}
+\alpha\sum_{j=0}^{N-1}r_\alpha^j\xi(Z_{-1-j}).
\]
Iterating $|X_{k+1}|\le r_\alpha|X_k|+2\alpha$ gives $|X_0|\le r_\alpha^N|X_{-N}|+2(1-r_\alpha^N)$. For every $\delta>0$, stationarity then gives
\[
\Pp(|X_0|>2+\delta)
\le\Pp(|X_0|>2+\delta r_\alpha^{-N})\longrightarrow0.
\]
Letting $N\to\infty$ shows that $|X_0|\le2$ almost surely. By stationarity, the same bound holds for every $X_{-N}$. Hence $r_\alpha^NX_{-N}\to0$ almost surely. The remaining series converges absolutely, and
\begin{equation}
Y_\alpha=\frac{X_0}{\sqrt\alpha}
=\sqrt\alpha\sum_{j=0}^{\infty}r_\alpha^j\zeta_j
\qquad\text{almost surely},
\label{eq:lower-series-app}
\end{equation}
where we set $\varepsilon_j=\eta_{-1-j}$. The reversed variables remain independent Rademacher variables, and
\begin{equation}
\xi(Z_{-1-j})
=\frac{\eta_{-1-j}(3-\eta_{-2-j})}{2}
=\frac{\varepsilon_j(3-\varepsilon_{j+1})}{2}
=: \zeta_j.
\end{equation}
The sequence $(\zeta_j)$ is $1$-dependent because $\zeta_j$ only uses $(\varepsilon_j,\varepsilon_{j+1})$.

Time reversal changes the orientation of the mixed moment. The original quantity
\[
\E[\xi(Z_0)\xi(Z_1)^2]
\]
appears as $\E[\zeta_j^2\zeta_{j+1}]$ in the backward stationary series. This is the nonzero orientation in the third-moment calculation below.

The marginal moments in \eqref{eq:lower-marginal-moments} also hold for $\zeta_j$. If $j\ne k$, then $\E[\zeta_j\zeta_k]=0$. For $|j-k|\ge2$, this follows from independence and centering. For adjacent indices, the Rademacher variable appearing only in the variable at the left endpoint has mean zero. The variance of \eqref{eq:lower-series-app} is consequently
\begin{align}
\E[Y_\alpha^2]
&=\alpha\sum_{j=0}^{\infty}r_\alpha^{2j}\E\zeta_j^2
=\frac52\frac{\alpha}{1-r_\alpha^2}
=\frac{5}{2(2-\alpha)}
=\frac54+O(\alpha).
\label{eq:lower-var-exact}
\end{align}

For the third moment, only configurations with two equal indices and one adjacent index can contribute. There are two orientations. Direct calculation gives
\begin{equation}
\E[\zeta_j^2\zeta_{j+1}]=-\frac94,
\qquad
\E[\zeta_j\zeta_{j+1}^2]=0.
\end{equation}
For the first identity, independence separates the expectation as
\[
\frac18\E\bigl[\varepsilon_{j+1}(3-\varepsilon_{j+1})^2\bigr]
\E[3-\varepsilon_{j+2}]
=-\frac94.
\]
The second identity is zero because $\varepsilon_j$ appears to the first power and is independent of the remaining factors. A product involving three distinct indices also has expectation zero. If the three indices are not consecutive, independence separates off a centered factor. If they are consecutive, the leftmost Rademacher variable appears only to the first power and again has conditional mean zero.

Expanding the cube in \eqref{eq:lower-series-app} and grouping the two possible adjacent repeated-index patterns gives
\begin{align}
\E[Y_\alpha^3]
&=3\alpha^{3/2}\sum_{j=0}^{\infty}
\left[
 r_\alpha^{3j+1}\E(\zeta_j^2\zeta_{j+1})
+r_\alpha^{3j+2}\E(\zeta_j\zeta_{j+1}^2)
\right]\notag\\
&=-\frac{27}{4}\alpha^{3/2}
\sum_{j=0}^{\infty}r_\alpha^{3j+1}
=-\frac{27}{4}\frac{\alpha^{3/2}r_\alpha}{1-r_\alpha^3}.
\label{eq:lower-third-exact}
\end{align}
Since $1-r_\alpha^3=3\alpha-3\alpha^2+\alpha^3$,
\begin{equation}
\E[Y_\alpha^3]
=-\frac94\sqrt\alpha+O(\alpha^{3/2}).
\end{equation}
The filtered sum therefore has an order-$\sqrt\alpha$ third moment while every individual observation has a symmetric law.

\subsection{Elementary characteristic-function recursion}

Each $\zeta_j$ depends on two adjacent Rademacher variables. After the variables up to $\varepsilon_j$ have been processed, only the value of $\varepsilon_j$ is needed for the next factor. This two-state memory yields an explicit recursion for the characteristic function.

\begin{lemma}[Two-state characteristic-function expansion]
\label{lem:lower-cf-recursion}
For the stationary series in \eqref{eq:lower-series-app}, the characteristic function is nonzero at frequency one for all sufficiently small $\alpha$. Here $\mathrm{i}=\sqrt{-1}$. With the logarithm defined by the near-identity product in the proof,
\begin{equation}
\log\E[e^{\mathrm{i}Y_\alpha}]
=-\frac12\E[Y_\alpha^2]
-\frac{\mathrm{i}}6\E[Y_\alpha^3]
+O(\alpha)
\qquad\text{as }\alpha\downarrow0.
\label{eq:lower-cf-expansion}
\end{equation}
\end{lemma}

\begin{proof}
Set
\[
\theta_j=\sqrt\alpha\,r_\alpha^j,
\qquad
Y_{\alpha,N}=\sum_{j=0}^{N-1}\theta_j\zeta_j.
\]
We integrate the Rademacher variables from left to right while retaining the terminal sign. Define
\begin{align*}
u_j^+
&=\E\!\left[
 e^{\mathrm{i}\sum_{k=0}^{j-1}\theta_k\zeta_k}
 \ind_{\{\varepsilon_j=1\}}
\right],\\
u_j^-
&=\E\!\left[
 e^{\mathrm{i}\sum_{k=0}^{j-1}\theta_k\zeta_k}
 \ind_{\{\varepsilon_j=-1\}}
\right],
\end{align*}
and set
\[
s_j=u_j^++u_j^-,
\qquad
d_j=u_j^+-u_j^-.
\]
Since $u_0^+=u_0^-=1/2$, we have $s_0=1$ and $d_0=0$. Also,
$s_N=\E[e^{\mathrm{i}Y_{\alpha,N}}]$.

Condition on $\varepsilon_j$ and then average over the fresh sign $\varepsilon_{j+1}$. If $\varepsilon_{j+1}=1$, then
$\zeta_j=\varepsilon_j$ and the new factor is $e^{\mathrm{i}\theta_j\varepsilon_j}$. If $\varepsilon_{j+1}=-1$, then
$\zeta_j=2\varepsilon_j$ and the new factor is $e^{2\mathrm{i}\theta_j\varepsilon_j}$. Therefore
\begin{align*}
u_{j+1}^+
&=\frac12\left(e^{\mathrm{i}\theta_j}u_j^+
+e^{-\mathrm{i}\theta_j}u_j^-\right),\\
u_{j+1}^-
&=\frac12\left(e^{2\mathrm{i}\theta_j}u_j^+
+e^{-2\mathrm{i}\theta_j}u_j^-\right).
\end{align*}
Adding and subtracting gives
\begin{align}
s_{j+1}
&=a(\theta_j)s_j+\mathrm{i}b(\theta_j)d_j,
\label{eq:lower-s-rec}\\
d_{j+1}
&=c(\theta_j)s_j+\mathrm{i}e(\theta_j)d_j,
\label{eq:lower-d-rec}
\end{align}
where
\begin{align*}
a(t)&=\frac{\cos t+\cos(2t)}2,
&b(t)&=\frac{\sin t+\sin(2t)}2,\\
c(t)&=\frac{\cos t-\cos(2t)}2,
&e(t)&=\frac{\sin t-\sin(2t)}2.
\end{align*}
For $|t|\le1/2$, Taylor's theorem gives uniform expansions
\begin{align}
a(t)&=1-\frac54t^2+O(t^4),
&b(t)&=\frac32t+O(t^3),
\label{eq:lower-ab-exp}\\
c(t)&=\frac34t^2+O(t^4),
&e(t)&=-\frac12t+O(t^3).
\label{eq:lower-ce-exp}
\end{align}
All $O(\cdot)$ constants below are uniform in $j$ and in sufficiently small $\alpha$.

We now control the normalized terminal-state imbalance. Whenever $s_j\ne0$, let
\[
\chi_j=\frac{d_j}{s_j}.
\]
Dividing \eqref{eq:lower-d-rec} by \eqref{eq:lower-s-rec} gives
\begin{equation}
\chi_{j+1}
=\frac{c(\theta_j)+\mathrm{i}e(\theta_j)\chi_j}
{a(\theta_j)+\mathrm{i}b(\theta_j)\chi_j}.
\label{eq:lower-chi-rec}
\end{equation}
We first establish a coarse bound that justifies every division. Choose $\alpha$ small enough that $\theta_j\le\sqrt\alpha\le1/2$ and $r_\alpha^{-1}\le2$. The Taylor bounds imply that, for a numerical constant $C_2$ and every $|t|\le1/2$,
\[
|a(t)-1|\le C_2t^2,
\qquad
|b(t)|+|e(t)|\le C_2|t|,
\qquad
|c(t)|\le C_2t^2.
\]
Set $C_\chi=4C_2$. We claim that
\begin{equation}
|\chi_j|\le C_\chi\theta_{j-1}^2,
\qquad
|a(\theta_j)+\mathrm{i}b(\theta_j)\chi_j|\ge\frac12,
\qquad j\ge1.
\label{eq:lower-chi-coarse}
\end{equation}
For the base case, $\chi_0=0$, so $|a(\theta_0)|\ge1-C_2\theta_0^2$. Thus $\chi_1$ is well defined and $|\chi_1|\le2C_2\theta_0^2\le C_\chi\theta_0^2$ once $\alpha$ is small enough.

It remains to verify the induction implication. Suppose $|\chi_j|\le C_\chi\theta_{j-1}^2$ and put $t=\theta_j$. Since $\theta_{j-1}\le2t$,
\begin{align*}
|c(t)+\mathrm{i}e(t)\chi_j|
&\le C_2t^2+4C_2C_\chi t^3,\\
|a(t)+\mathrm{i}b(t)\chi_j|
&\ge1-C_2t^2-4C_2C_\chi t^3.
\end{align*}
After decreasing the upper bound on $\alpha$, the second quantity is at least $1/2$, and twice the first is at most $C_\chi t^2$. Equation~\eqref{eq:lower-chi-rec} then gives $|\chi_{j+1}|\le C_\chi\theta_j^2$. Applying this implication first at $j=1$ and then successively proves \eqref{eq:lower-chi-coarse}. Since $s_{j+1}=s_j[a(\theta_j)+\mathrm{i}b(\theta_j)\chi_j]$ and $s_0=1$, it also proves that every $s_j$ is nonzero.

The coarse estimate allows us to sharpen the real and imaginary parts. We claim that
\begin{equation}
\Re\chi_j=\frac34\theta_{j-1}^2+O(\theta_{j-1}^4),
\qquad
\Im\chi_j=O(\theta_{j-1}^3),
\qquad j\ge1.
\label{eq:lower-chi-exp}
\end{equation}
For $j=1$, the ratio $\chi_1=c(\theta_0)/a(\theta_0)$ is real, and \eqref{eq:lower-chi-exp} follows from \eqref{eq:lower-ab-exp}--\eqref{eq:lower-ce-exp}.

Assume \eqref{eq:lower-chi-exp} at index $j$ and write
$\chi_j=x_j+\mathrm{i}y_j$. Put $t=\theta_j$. Since
$\theta_{j-1}=t/r_\alpha$ and $r_\alpha^{-1}\le2$,
\[
x_j=O(t^2),
\qquad
y_j=O(t^3).
\]
The numerator and denominator of \eqref{eq:lower-chi-rec} become
\begin{align*}
c(t)+\mathrm{i}e(t)\chi_j
&=c(t)-e(t)y_j+\mathrm{i}e(t)x_j\\
&=\frac34t^2+O(t^4)+\mathrm{i}O(t^3),\\
a(t)+\mathrm{i}b(t)\chi_j
&=a(t)-b(t)y_j+\mathrm{i}b(t)x_j\\
&=1+O(t^2)+\mathrm{i}O(t^3).
\end{align*}
The reciprocal of the denominator is
$1+O(t^2)+\mathrm{i}O(t^3)$. Multiplication gives
\[
\Re\chi_{j+1}=\frac34t^2+O(t^4),
\qquad
\Im\chi_{j+1}=O(t^3),
\]
which proves \eqref{eq:lower-chi-exp}.

Define the one-step characteristic-function factor
\[
g_j=\frac{s_{j+1}}{s_j}
=a(\theta_j)+\mathrm{i}b(\theta_j)\chi_j.
\]
For $j=0$, use $\chi_0=0$. For $j\ge1$, substitute
\eqref{eq:lower-chi-exp} into \eqref{eq:lower-ab-exp}. Since
$\mathrm{i}b\chi=-b\Im\chi+\mathrm{i}b\Re\chi$, we obtain
\begin{equation}
g_j-1
=-\frac54\theta_j^2
+\frac{9\mathrm{i}}8
\ind_{\{j\ge1\}}\theta_j\theta_{j-1}^2
+O(\theta_j^4).
\label{eq:lower-g-exp}
\end{equation}
The real correction from $-b\Im\chi$ is $O(\theta_j^4)$. The leading imaginary term is
$(3\theta_j/2)(3\theta_{j-1}^2/4)$.

For small $\alpha$, \eqref{eq:lower-g-exp} gives $|g_j-1|\le1/2$ uniformly in $j$. Let $\log g_j$ denote the principal logarithm on this disk. Since
$\log(1+w)=w+O(|w|^2)$,
\[
\log g_j
=-\frac54\theta_j^2
+\frac{9\mathrm{i}}8
\ind_{\{j\ge1\}}\theta_j\theta_{j-1}^2
+O(\theta_j^4).
\]
Because $s_0=1$ and $s_N=\prod_{j=0}^{N-1}g_j$, define
\[
\log s_N=\sum_{j=0}^{N-1}\log g_j.
\]
The series of logarithms converges absolutely because
$\sum_j\theta_j^2<\infty$. Its limit defines the logarithm in the statement of the lemma, and the corresponding infinite product is nonzero. Also,
$Y_{\alpha,N}\to Y_\alpha$ almost surely. Bounded convergence gives
$s_N\to\E[e^{\mathrm{i}Y_\alpha}]$. Letting $N\to\infty$ yields
\begin{align}
\log\E[e^{\mathrm{i}Y_\alpha}]
&=-\frac54\sum_{j\ge0}\theta_j^2
+\frac{9\mathrm{i}}8\sum_{j\ge1}\theta_j\theta_{j-1}^2
+O\!\left(\sum_{j\ge0}\theta_j^4\right).
\label{eq:lower-log-series}
\end{align}
The geometric sums are
\begin{equation}
\sum_{j\ge0}\theta_j^2=\frac1{2-\alpha},
\qquad
\sum_{j\ge1}\theta_j\theta_{j-1}^2
=\frac{\alpha^{3/2}r_\alpha}{1-r_\alpha^3},
\qquad
\sum_{j\ge0}\theta_j^4=O(\alpha).
\end{equation}
Equation~\eqref{eq:lower-var-exact} identifies the first term in
\eqref{eq:lower-log-series} as $-\E[Y_\alpha^2]/2$. Equation~\eqref{eq:lower-third-exact} identifies the second as
$-\mathrm{i}\E[Y_\alpha^3]/6$. This proves \eqref{eq:lower-cf-expansion}.
\end{proof}

Apply Lemma~\ref{lem:lower-cf-recursion} and substitute the exact moments
\eqref{eq:lower-var-exact} and \eqref{eq:lower-third-exact}. We obtain
\begin{align}
\log\E[e^{\mathrm{i}Y_\alpha}]
&=-\frac12\frac{5}{2(2-\alpha)}
+\frac{9\mathrm{i}}8
\frac{\alpha^{3/2}r_\alpha}{1-r_\alpha^3}
+O(\alpha)\notag\\
&=-\frac58+\frac{3\mathrm{i}}8\sqrt\alpha+O(\alpha).
\end{align}
Exponentiating gives
\begin{equation}
\E[e^{\mathrm{i}Y_\alpha}]
=e^{-5/8}\left(1+\frac{3\mathrm{i}}8\sqrt\alpha+O(\alpha)\right).
\end{equation}
Taking imaginary parts yields
\begin{equation}
\E[\sin(Y_\alpha)]
=\frac38e^{-5/8}\sqrt\alpha+O(\alpha).
\label{eq:lower-sine-app}
\end{equation}

\subsection{Wasserstein test function and independent-sampling comparison}

Let $G\sim\cN(0,5/4)$. The function $h(x)=\sin x$ is $1$-Lipschitz, and symmetry gives $\E h(G)=0$. Kantorovich--Rubinstein duality and \eqref{eq:lower-sine-app} imply
\[
\Wone(\Law(Y_\alpha),\Law(G))
\ge|\E\sin(Y_\alpha)|.
\]
Thus there exists $\alpha_0>0$ such that, for every $0<\alpha\le\alpha_0$,
\[
|\E\sin(Y_\alpha)|
\ge\frac{3}{16}e^{-5/8}\sqrt\alpha.
\]
This proves Theorem~\ref{thm:lower} with $c_0=\tfrac{3}{16}e^{-5/8}$.

It remains to make the dependence comparison explicit. Let $(\widetilde\zeta_j)_{j\ge0}$ be independent with the same marginal law as $\zeta_0$, and define
\[
\widetilde Y_\alpha
=\sqrt\alpha\sum_{j=0}^{\infty}r_\alpha^j\widetilde\zeta_j.
\]
Independence and symmetry make the full sequence jointly invariant under global sign reversal. Hence $\widetilde Y_\alpha$ is symmetric and $\E\sin(\widetilde Y_\alpha)=0$ for every $\alpha\in(0,1)$. Moreover, \eqref{eq:lower-cov-zero-app} shows that $Y_\alpha$ and $\widetilde Y_\alpha$ have the same variance, so they share the same limiting Gaussian comparator. The Markov construction has the order-$\sqrt\alpha$ sine term generated by \eqref{eq:lower-mixed-app}, while any jointly sign-symmetric construction has zero sine expectation.

\subsection{Interpretation of the lower bound}

The quadratic objective $f(x)=x^2/2$ has exactly linear drift, so the nonlinear Taylor remainder from the upper-bound proof is absent. The order-$\sqrt\alpha$ discrepancy therefore comes entirely from temporal dependence. Nevertheless, the Markovian and independent models agree in their one-time noise law and second-order structure. Both have the symmetric marginal distribution on $\{-2,-1,1,2\}$ and long-run covariance $\Sigma_M=5/2$, and every nonzero-lag covariance in the Markovian model is zero. Thus both have the limiting Gaussian law $\cN(0,5/4)$.

The distinction appears in the adjacent mixed moment
\[
\E[\xi(Z_0)\xi(Z_1)^2]=-\frac94.
\]
This moment survives the geometric filtering in the stationary recursion and produces an order-$\sqrt\alpha$ third cumulant. The characteristic-function recursion converts that local temporal asymmetry into the following expectation:
\[
\E[\sin(Y_\alpha)]
=\frac38e^{-5/8}\sqrt\alpha+O(\alpha).
\]
Independent sampling has the same covariance calculation, while joint sign symmetry makes its sine expectation exactly zero.

The scale can also be read directly from the exact third-moment sum. Each adjacent repeated-index pattern has weight of order $\alpha^{3/2}r_\alpha^{3j}$, and the number of effectively contributing indices is of order $\alpha^{-1}$. Equivalently,
\[
\alpha^{3/2}\sum_{j\ge0}r_\alpha^{3j}
=\frac{\alpha^{3/2}}{1-r_\alpha^3}
=O(\!\sqrt\alpha).
\]
The two-state recursion transfers this scale from an unbounded third-moment calculation to the bounded sine test function required by Wasserstein duality.

The coefficient $3e^{-5/8}/8$ in \eqref{eq:lower-sine-app} combines the order-$\sqrt\alpha$ imaginary correction $3/8$ with the Gaussian damping factor $e^{-5/8}$.

The two proofs separate three effects. Strong convexity gives block contraction, Hessian regularity controls nonlinear drift, and the Poisson correctors encode temporal dependence. Under independent sampling, $\cC(z)=\Sigma_M$ and $W=0$. The lower example shows that higher-order dependence can still affect the order-$\sqrt\alpha$ correction. Thus long-run covariance determines the Gaussian limit, while the first correction may also reflect objective curvature and higher-order temporal structure.

\end{document}